\documentclass{article}
\usepackage{arxiv_style,times}

\usepackage{amsmath,amsfonts,bm}

\def\eqref#1{equation~\ref{#1}}

\def\1{\bm{1}}

\DeclareMathAlphabet{\mathsfit}{\encodingdefault}{\sfdefault}{m}{sl}
\SetMathAlphabet{\mathsfit}{bold}{\encodingdefault}{\sfdefault}{bx}{n}

\arxivcopy

\usepackage{amsmath}
\usepackage{amssymb}
\usepackage{mathtools}
\usepackage{amsthm}

\usepackage{algorithm}
\usepackage{algpseudocode}
\usepackage{hyperref}
\usepackage[capitalize,nameinlink,noabbrev]{cleveref}
\usepackage{url}
\usepackage{booktabs}
\usepackage{multirow}
\usepackage{graphicx}
\usepackage{wrapfig}
\usepackage{tikz}
\usepackage{subcaption}
\usepackage{enumitem}
\usepackage{lipsum}

\crefformat{equation}{Equation~#2#1#3}
\Crefformat{equation}{Equation~#2#1#3}

\usepackage{thmtools}
\usepackage{thm-restate}
\usepackage{thmtools}
\theoremstyle{plain}
\newtheorem{theorem}{Theorem}[section]

\newtheorem{lemma}[theorem]{Lemma}
\newtheorem{corollary}[theorem]{Corollary}
\theoremstyle{definition}
\newtheorem{definition}[theorem]{Definition}
\newtheorem{assumption}[theorem]{Assumption}

\theoremstyle{remark}

\usetikzlibrary{arrows.meta, positioning, shapes.multipart}

\usepackage[disable]{todonotes}

\title{DAG-FM: A Foundation Model for Causal Discovery under Heterogeneous Causal Mechanisms}

\author{
\textbf{Yikang Chen$^1$, Zhengkang Guan$^1$, Haoyuan Qian$^1$, Xingxuan Zhang$^3$, Peng Cui$^2$, } \\
\textbf{\ Yi Yang$^1$, Fei Wu$^1$, Kun Kuang$^1$}\thanks{Corresponding author: \texttt{kunkuang@zju.edu.cn}} \\[1mm]
\normalsize $^1$Zhejiang University \quad 
$^2$Tsinghua University \quad 
$^3$Stable AI
}

\begin{document}
\maketitle

\begin{abstract}
Causal discovery from observational tabular data remains fundamentally challenging, primarily due to the heterogeneity of underlying causal mechanisms and the high-dimensional combinatorial search space of Directed Acyclic Graphs (DAGs). In this paper, we propose \textbf{DAG-FM}, a novel foundation model architecture that amortizes causal discovery. Unlike direct matrix prediction, DAG-FM decomposes the causal discovery process into two auto-regressive stages using two specialized Transformer-based sub-modules: a leaf-node predictor and a parent-node predictor. To effectively model complex row-column interactions, we adopt a robust tabular interaction block to output feature-wise representations. Crucially, to handle diverse and unknown Functional Causal Model (FCM) assumptions in real-world scenarios, we introduce Mixture-of-Leaf-Experts (MoLE), allowing the model to dynamically route and adapt to identifiable mechanism families. Through an iterative inference algorithm, DAG-FM seamlessly extracts causal orderings and constructs valid DAGs. Extensive experiments demonstrate that DAG-FM achieves state-of-the-art performance on both synthetic benchmarks and complex real-world datasets, significantly outperforming traditional classical algorithms and recent foundation models in both accuracy and scalability.
\end{abstract}
\section{Introduction}

Causal discovery aims to recover causal structures from data \citep{Pearl2009}, playing a pivotal role in diverse fields such as bioinformatics \citep{Zhang2013}, epidemiology \citep{Vandenbroucke2016}, sociology \citep{Huber2024}, and manufacturing \citep{Vukovic2022}. A primary objective of this task is to identify the underlying Directed Acyclic Graphs (DAGs) from observational data. Existing approaches, spanning constraint-based \citep{Spirtes1991, Spirtes1995}, score-based \citep{Cooper1992, Chickering2002}, and optimization-based \citep{Zheng2018, Bello2022} methods, have been extensively developed. However, without additional assumptions, these methods can only identify causal structures up to Markov equivalence classes. Functional Causal Models (FCMs) introduce supplementary structural assumptions into Structural Causal Models (SCMs), ensuring theoretical identifiability of the underlying DAGs from observational distributions when specific conditions are met. Notable examples include Linear Non-Gaussian Acyclic Models (LiNGAMs, \citet{Shimizu2006}), Additive Noise Models (ANMs, \citet{Hoyer2008}), Heteroscedastic Noise Models (HNMs, \citet{Tagasovska2020}), and Post-Nonlinear models (PNLs, \citet{Zhang2009}).

Building upon the theoretical guarantees of DAG identifiability, FCM-based causal discovery methods have been developed for empirical DAG recovery. However, these methods typically assume homogeneous causal mechanisms, fundamentally limiting their generalizability. To render statistical hypothesis testing tractable, most existing approaches are narrowly tailored to specific theoretical frameworks, such as LiNGAM or ANM \citep{Shimizu2006,Shimizu2011,Peters2014,Rolland2022}. Alternatively, some methods derive causal orderings by extracting specific statistical features, yet they remain effective only under highly restrictive conditions \citep{Reisach2021,Reisach2023}. Consequently, the robustness and validity of these traditional methods under model misspecification or within highly complex real-world scenarios remain largely questionable.

In parallel, amortized algorithms have emerged for causal discovery, empirically demonstrating consistent superiority over traditional methods in complex nonlinear data regimes \citep{Lorch2022, Dhir2025, Thompson2026, Peng2026, Li2026}. These approaches frame causal discovery as a table-to-graph prediction task, leveraging a foundation model pre-trained on synthetic data generated from complex prior spaces to predict causal graphs for unseen data. However, the correctness and performance of this learning paradigm are highly sensitive to the design of the prior space. For example, in bivariate causal discovery, \citet{Dhir2025} showed that if the prior space is non-identifiable, the output is ambiguous; conversely, \citet{Montagna2025} demonstrated that models trained on well-specified, identifiable prior spaces can achieve consistent results.

To address both the limited generalizability of FCM-based methods under heterogeneous mechanisms and the lack of theoretical guarantees in amortized inference, we propose integrating these two paradigms. To this end, we introduce \textbf{DAG-FM}, a foundation model for causal discovery that intrinsically guarantees the output of strictly identifiable DAGs solely from observational data. 

Specifically, our contributions bridge theory and practice:
Theoretically, we establish a design condition detailing how to heterogeneously incorporate FCM assumptions into the prior space. This formulation guarantees that the amortized model converges to a unique, identifiable DAG. 
Empirically, we decouple the DAG identification process into two synergistic sub-tasks: \textit{leaf-node prediction} and \textit{parent-node prediction}, conducting large-scale pre-training for each. Together, these two learned modules constitute DAG-FM, collaborating sequentially to reconstruct the DAG. 

This architecture enables highly scalable inference, allowing DAG-FM to handle datasets with up to $100{,}000$ samples or $1{,}000$ dimensions (or $30{,}000$ samples and $100$ dimensions simultaneously) on a single GPU with 24GB VRAM under half-precision \texttt{BP16}. Extensive experiments demonstrate that DAG-FM achieves state-of-the-art performance, significantly outperforming traditional approaches on both synthetic and real-world datasets, while also surpassing contemporary causal discovery foundation models across multiple evaluation metrics.

In summary, our primary contributions are as follows:
\vspace{-0.5em}
\begin{itemize}[leftmargin=*]
\itemsep 0em
\item We provide a design condition for the prior space that allows for heterogeneous causal mechanisms, while theoretically ensuring that the predictions of the amortized method converge to a unique DAG from observational data.
\item We develop \textbf{DAG-FM}, a foundation model for causal discovery consisting of two sub-modules, which guarantees the accurate prediction of a unique DAG directly from observational data.
\item We provide extensive experiments demonstrating that DAG-FM achieves state-of-the-art performance on both synthetic and real-world benchmark tests.
\end{itemize}
\section{Related Works}

\paragraph{FCM-based Causal Discovery}
FCM-based causal discovery aims to recover the underlying causal structure by restricting the causal mechanisms within a SCM. Theoretically, to achieve the ultimate goal of DAG identifiability, FCM approaches impose specific structural constraints on the underlying causal mechanisms. For instance, LiNGAM \citep{Shimizu2006} elegantly established that causal directions are uniquely identifiable under the assumptions of linear causal relationships and non-Gaussian additive noise. Subsequently, \citet{Hoyer2008} extended this identifiability to ANM with nonlinear causal mechanisms and general noise distributions. Building upon this, \citet{Zhang2009} further generalized the identifiability conditions to post-nonlinear settings, characterizing an even broader space of identifiable models. More recently, these identifiability conditions have been further relaxed to accommodate heteroscedastic noises \citep{Immer2023,Strobl2023}, while contemporary theoretical efforts \citep{Xi2025,Chen2026} strive to generalize FCM framework to even broader conditions.

In practice, a corresponding lineage of multivariate causal discovery algorithms has been explicitly developed to translate these theoretical guarantees into empirical DAG recovery. For example, within the LiNGAM framework, algorithms such as Direct-LiNGAM \citep{Shimizu2011} and ICA-LiNGAM \citep{Shimizu2006} are capable of identifying the entire DAG. For ANMs, a diverse suite of methods has been deployed, encompassing RESIT \citep{Peters2014}, CAM \citep{Buhlmann2014}, SCORE \citep{Rolland2022}, DAS \citep{Montagna2023b}, and NoGAM \citep{Montagna2023}. Similarly, addressing HNMs, recent algorithmic developments include SkewScore \citep{Lin2025} and HOST \citep{Duong2022}. However, these classical methods heavily premise on the homogeneity of the imposed mechanisms. Consequently, their robustness in handling heterogeneous mechanisms or severe model misspecifications remain largely uncertain.

\paragraph{Amortized Causal Discovery}
Amortized causal discovery employs a large structured-data model (LDM) to predict causal graphs directly from tabular data by treating structured data as joint distributions \citep{Zhang2025b}. This predictive model is pre-trained by sampling random causal graphs along with their corresponding observations from an artificially constructed prior space. Unlike tabular foundation models, which utilize in-context learning to predict missing entries at the cell level, amortized causal discovery operates at the variable level to infer the causal relationships among variables.

For instance, AVICI \citep{Lorch2022} employs an Axial Transformer to model interactions across tabular features, aggregates feature representations through max-pooling, and utilizes a bilinear head to predict the probability of edge existence in the adjacency matrix. Similarly, CSIvA \citep{Ke2023} predicts edge probabilities autoregressively. Related approaches, such as CauScale \citep{Peng2026} and TabCausal \citep{Li2026}, further extend this paradigm by scaling the training and inference processes to larger regimes.

However, a significant challenge with these methods is their direct modeling of the adjacency matrix, which fails to formally guarantee that the output is a DAG. BCNP \citep{Dhir2025} and Arrow \citep{Thompson2026} address this deficiency by decomposing the prediction task into causal ordering and an upper-triangular matrix. Other strategies focus on alternative forms of causal structures: for instance, SiCL \citep{Zhang2025} aims to predict Markov equivalence classes, TCD-DL \citep{Kim2025} focuses on predicting the ancestor set for a target node, and TabOrder \citep{Xu2026} is dedicated to predicting the causal ordering of variables.
\section{Preliminaries}

\paragraph{Structural Causal Model}
An SCM involving $d$ variables is defined as a triplet $\phi = (\boldsymbol{f}, \boldsymbol{X}, \boldsymbol{U})$, where $\boldsymbol{X}$ and $\boldsymbol{U}$ denote exogenous and endogenous variables, respectively. For each $i \in I = \{1, \dots, d\}$, the relationship is defined as $X_i = f_i(\boldsymbol{X}_{\text{pa}(i)}, U_i)$, where $\text{pa}(i) \subseteq I \setminus \{i\}$ denotes the set of parents of $X_i$, and $f_i$ represents the causal mechanism. We assume the SCM is Markovian \citep{Pearl2009}, implying $\boldsymbol{X}_{\text{pa}(i)} \perp\!\!\!\perp U_i$ for all $X_i$. The graph $\mathcal{G} = (I, \mathcal{E})$, where $(j, i) \in \mathcal{E}$ if and only if $j \in \text{pa}(i)$, is known as the causal graph and is a DAG. A topological ordering $\tau$ of this graph defines a causal ordering.

\paragraph{Amortized Causal Discovery}
Let $p(\phi)$ be a prior distribution over an artificially constructed prior space $\Phi$ of SCMs. An SCM $\phi \sim p(\phi)$ with $d$ variables corresponds to a specific DAG $\mathcal{G}$. Observational data $\mathcal{D} = \{\boldsymbol{x}_k\}_{k=1}^n \in \mathbb{R}^{n \times d}$ is sampled from the observational distribution $p(\boldsymbol{X}\!\mid\!\phi)$. The posterior distribution of the causal graph given the observational data is defined as:
\begin{align}
P(\mathcal{G}\!\mid\!\mathcal{D}) \propto \int P(\mathcal{G}\!\mid\!\mathcal{D}, \phi) \, p(\mathcal{D}\!\mid\!\phi) \, p(\phi) \, d\phi.
\end{align}
To predict $\mathcal{G}$ given $\mathcal{D}$, an amortized model $Q_\theta(\mathcal{G}\!\mid\!\mathcal{D})$ is constructed to approximate the posterior distribution $P(\mathcal{G}\!\mid\!\mathcal{D})$. This is achieved by minimizing the $\text{KL}(P \| Q_\theta)$ divergence, which is equivalent to maximizing the log-likelihood $\mathbb{E}_{\mathcal{G}, \mathcal{D}}[\log Q_\theta(\mathcal{G}\!\mid\!\mathcal{D})]$. During training, we perform large-scale sampling of random SCMs $\phi$ from $p(\phi)$ to generate corresponding pairs of $(\mathcal{G}, \mathcal{D})$ for pre-training the model $Q_\theta$. During inference, the model takes $\mathcal{D}$ as input and produces $\mathcal{G}$ through either sampling or maximum a posteriori (MAP) estimation.
\section{DAG-FM: The DAG Foundation Model}
\subsection{DAG Identifiability via Mechanism Families}
The asymptotic identifiability of DAGs in amortized methods hinges on whether the MAP estimator $\arg\max_{\mathcal{G}}Q_\theta(\mathcal{G}\!\mid\!\mathcal{D})$ converges to a unique ground-truth graph as the model parameters $\theta$ are optimized and the sample size $|\mathcal{D}|$ increases. By the properties of KL divergence and the Law of Large Numbers, we have:
\begin{align}
\lim_{|\mathcal{D}|\to\infty}Q_{\theta^*}(\mathcal{G}\!\mid\!\mathcal{D})=\lim_{|\mathcal{D}|\to\infty}P(\mathcal{G}\!\mid\!\mathcal{D})=P(\mathcal{G}\!\mid\!p(\boldsymbol{X})),
\end{align}
where $\theta^* = \arg\min_\theta \text{KL}(P\|Q_\theta)$ denotes the optimal parameters under which $Q_\theta(\mathcal{G}\!\mid\!\mathcal{D})$ matches the posterior $P(\mathcal{G}\!\mid\!\mathcal{D})$. As $|\mathcal{D}|\to\infty$, the empirical distribution $\hat{p}_{\mathcal{D}}(\boldsymbol{X})=\sum_{i=1}^{|\mathcal{D}|}\delta_{\boldsymbol{x}_i}(\boldsymbol{X})$ converges almost surely to the true observational distribution $p(\boldsymbol{X})$. Consequently, the convergence of the amortized model to a unique graph is equivalent to the uniqueness of the posterior map $\arg\max_{\mathcal{G}}P(\mathcal{G}\!\mid\!p(\boldsymbol{X}))$.

\begin{definition}[Almost Sure Identifiability]
A DAG $\mathcal{G}$ is said to be identifiable almost surely in the posterior if there exists a graph $\mathcal{G}$ such that $P(\mathcal{G}\!\mid\!p(\boldsymbol{X}))=1$.
\end{definition}

To guarantee this, we incorporate assumptions rooted in established FCM literature, focusing on four primary classes of causal mechanisms:

\begin{definition}[Mechanism Families $\mathcal{F}$]
\label{def:1}
For a mechanism $f(\boldsymbol{X}_{\text{pa}(i)}, U_i)$, we define:
\vspace{-0.5em}
\begin{itemize}
\itemsep0em
\item $\mathcal{F}_{\text{LiNGAM}}$: Specified by a vector $\boldsymbol{w}\in\mathbb{R}^d$ such that $f(\boldsymbol{X},U)=\boldsymbol{w}^\intercal\boldsymbol{X}+U$, where $U_i$ is a non-Gaussian variable.
\item $\mathcal{F}_{\text{ANM}}$: Specified by a non-linear function $g:\mathbb{R}^d\to\mathbb{R}$ such that $f(\boldsymbol{X},U)=g(\boldsymbol{X})+U$.
\item $\mathcal{F}_{\text{HNM}}$: Specified by a function $g:\mathbb{R}^d\to\mathbb{R}$ and a non-constant function $h:\mathbb{R}^d\to\mathbb{R}_+$ such that $f(\boldsymbol{X},U)=g(\boldsymbol{X})+h(\boldsymbol{X})\,U$.
\item $\mathcal{F}_{\text{PNL}}$: Specified by a function $g:\mathbb{R}^d\to\mathbb{R}$ and a monotonic non-linear function $h:\mathbb{R}\to\mathbb{R}$ such that $f(\boldsymbol{X},U)=h(g(\boldsymbol{X})+U)$.
\end{itemize}
\end{definition}
\vspace{0.15em}

\begin{assumption}[Prior Space]\label{assum:1}
(i) Causal sufficiency (no unobserved confounders) and faithfulness (conditional independencies imply d-separation) hold. (ii) For any $\phi = (\boldsymbol{X}, \boldsymbol{f}, p(\boldsymbol{U}))\in\Phi$, each mechanism $f_i$ belongs to the family $\mathcal{F}_i \in \{\mathcal{F}_{\text{LiNGAM}}, \mathcal{F}_{\text{ANM}}, \mathcal{F}_{\text{HNM}}, \mathcal{F}_{\text{PNL}}\}$.
\end{assumption}

In stark contrast to traditional FCM-based methods that assume homogeneous data generation, \cref{assum:1} permits the SCM to exhibit \textit{heterogeneous mechanisms}. That is, an SCM can be elegantly formulated as a composite mixture of mechanisms drawn from fundamentally different families. The following theoretical results guarantee that, within a prior space satisfying \cref{assum:1}, the true DAG is almost surely identifiable from the posterior.

\begin{restatable}{theorem}{theoremMechanims}\label{thm:1}
Under \cref{assum:1}, given that $Y$ is the effect variable, if $Y=f(\boldsymbol{X},U)$ induces the observational distribution $p(\boldsymbol{X},Y)$ and the family of the mechanism $f$ is $\mathcal{F}$, then the posterior probability $P(\mathcal{F}\!\mid\!p(\boldsymbol{X},Y))=1$.
\end{restatable}

In essence, under the heterogeneous setup delineated in \cref{assum:1}, Theorem \ref{thm:1} demonstrates that the exact nature of the mechanism family can be perfectly discriminated by the posterior, given only the observational distribution. 

\begin{restatable}{theorem}{theoremDAG}\label{thm:2}
Under \cref{assum:1} and some additional identifiability conditions\footnote{These additional identifiability conditions require that condition (iv) in \cref{assum:qpe_wronk} and \cref{assum:finite_basis} hold simultaneously.}, if the SCM $\phi=(\boldsymbol{X},\boldsymbol{f},p(\boldsymbol{U}))$ induces the observational distribution $p(\boldsymbol{X})$ and the causal graph $\mathcal{G}$, then the posterior probability is $P(\mathcal{G}\!\mid\!p(\boldsymbol{X}))=1$, i.e., $\mathcal{G}$ is identifiable almost surely.
\end{restatable}

In other words, we establish a sufficient condition for prior space design that guarantees DAG identifiability. As long as the prior space constrains randomly sampled causal mechanisms to belong to one of the four families in \cref{def:1} (and satisfies the additional identifiability conditions), while ensuring that causal sufficiency and faithfulness hold, the DAG is guaranteed to be identifiable under the posterior distribution. Furthermore, this theoretically ensures that the amortized model will inherently converge, mapping observational data to a uniquely identifiable DAG.
\subsection{Prior Space Construction}
\label{sec:4_prior}

To construct a prior space $\Phi$ that is both expressive and adheres to the theoretical requirements stipulated \cref{assum:1}, we design random SCMs using the following components.

\paragraph{Random Graphs}
We employ three generative graph models: Erdős–Rényi, Barabási–Albert, and Watts–Strogatz. The number of nodes $d \in \{2, \dots, 100\}$.

\paragraph{Random Noise}
Exogenous variables are sampled from a comprehensive library of distributions: Normal, Laplace, Uniform, Student’s $t$, Log-Normal, Gumbel, Exponential, $\chi^2$, Beta, and Gamma. Each distribution is parameterized by randomly sampled hyper-parameters. 

\paragraph{Random Mechanisms}
To implement the underlying causal mechanisms, we consider a diverse repertoire of random functions, including random linear functions, Multi-Layer Perceptrons (MLPs), periodic functions, Gaussian Processes with Random Fourier Features (GP-RFF), and Random Forests. In addition to the four strictly defined mechanism families outlined in \cref{def:1}, we introduce a generalized causal mechanism formally defined as $f(\boldsymbol{X}, U)$. This is implemented by treating the exogenous noise $U$ directly as an intrinsic input dimension during the forward pass. Although this general mechanism currently lacks a rigorous theoretical guarantee for strict DAG identifiability, we empirically observe that the implicit inductive biases of the function approximators still endow it with a non-trivial degree of practical identifiability. Consequently, for each local structure within the randomly sampled DAG, we uniformly sample from these mechanism families to instantiate the respective causal assignment.

To prevent the model from exploiting spurious statistical artifacts that could artificially trivialize DAG identification on synthetic benchmarks \citep{Reisach2021,Reisach2023}, we enforce strict distribution normalization. Specifically, immediately following each forward inference step of the structural equation $Y=f(\boldsymbol{X},U)$, the generated output $Y$ is dynamically standardized to have zero mean and unit variance \citep{Ormaniec2025}.

\paragraph{Post-processing}
Adopting the data augmentation pipeline from \cite{Hollmann2025}, we subject a fraction of the features to non-linear warping and discretization. Warping functions include sinh-arcsinh transformations, residual flows, and asymmetric power functions. Discretization is performed via quantile binning, uniform windowing, and ordinal encoding to simulate real-world data characteristics.

\paragraph{Sampling Tabular Data}
After generating the full feature set, we organize the observations into a tabular dataset $\mathcal{D} \in \mathbb{R}^{n \times d}$ and shuffle the feature order. Unlike \cite{Hollmann2025, Qu2026}, we perform no feature dropout to strictly satisfy the causal sufficiency assumption. While the training sample size is fixed at $n=1024$, our amortized model is designed to generalize to varying sample sizes during inference.
\subsection{Decomposition of DAG Identification}
While \cref{thm:2} guarantees the existence of a unique identifiable DAG under \cref{assum:1}, existing amortized methods like \cite{Lorch2022} often parameterize the adjacency matrix directly. This approach faces two primary hurdles: the lack of formal DAG constraints and the quadratic scaling of the output space with respect to dimensionality, which impedes scalability to high-dimensional datasets. Following \cite{Dhir2025, Thompson2026}, we decompose the DAG identification problem into two sub-problems: \textit{order identification} and \textit{DAG pruning}, formulated as:
\begin{align}
P(\mathcal{G}\!\mid\!p(\boldsymbol{X}))=\sum_{\tau} P(\mathcal{G}\!\mid\!\tau, p(\boldsymbol{X}))\,P(\tau\!\mid\!p(\boldsymbol{X})).
\end{align}
Each sub-problem is further decomposed into recursive prediction tasks.

\paragraph{Amortized Order Identification via Leaf Prediction}
Inspired by order-based causal discovery, we identify the causal ordering $\tau$ by recursively uncovering the leaves of the DAG. The task of identifying a leaf node $i$ from a set of variables $\boldsymbol{X}$ is framed as $\arg\max_i P(L_i=1\!\mid\!p(\boldsymbol{X}))$, where the event $L_i$ indicates whether $i$ is a leaf node in the DAG. This transformation reduces the ordering problem to a sequence of binary classification tasks.

\paragraph{Amortized DAG Pruning via Parent Prediction}
Given a causal ordering $\tau$, the DAG skeleton is constrained to an upper triangular structure. To avoid the quadratic growth of predicting the entire adjacency matrix, we recursively prune the graph by identifying the parents of each leaf node. Specifically, the identification of a parent $j$ for a confirmed leaf node $i$ is framed as $\arg\max P(E_{ji}\!\mid\!L_i=1, p(\boldsymbol{X}))$, where $E_{ji}$ denotes the presence of a directed edge $(j \to i)$. This process remains a tractable series of binary classification tasks.

\paragraph{Mixture of Experts for Leaf Prediction}

Since our approach relies on the theoretical conditions of FCM-based causal discovery under \cref{assum:1}, we analogously draw inspiration from classical FCM practices. Empirically, discovering varying mechanism families requires developing distinct, tailored algorithms; these identification methods exhibit specific focal preferences and generally operate flawlessly only under strict assumption matching. Motivated by this, to effectively handle heterogeneous causal mechanisms, we design a novel framework denoted as \textit{Mixture-of-Leaf-Expert (MoLE)}. This design encourages different experts to dynamically attend to distinct data features, thereby enhancing the holistic predictive capability of the model across each mechanism family.

We refine the leaf node prediction process into two hierarchical steps: (i) identifying the underlying mechanism family; (ii) predicting the leaf node conditioned on that family. This is formalized as:
\begin{align}
P(L_i=1\!\mid\!p(\boldsymbol{X}))=\sum_{\mathcal{F}_i} P(L_i=1\!\mid\!\mathcal{F}_i, p(\boldsymbol{X}))\,P(\mathcal{F}_i\!\mid\!p(\boldsymbol{X})),
\end{align}
where, according to \cref{thm:1}, $P(\mathcal{F}_i\!\mid\!p(\boldsymbol{X}))$ effectively acts as a degenerate distribution. This framework offers three primary advantages: 
First, it modularizes the amortized inference process by assigning specialized experts to distinct mechanism families. 
Second, it provides a degree of interpretability by explicitly surfacing the inferred mechanism class and its corresponding theoretical foundation. 
Third, this framework enables human intervention: one may manually adjust $P(\mathcal{F}_i\!\mid\!p(\boldsymbol{X}))$ to incorporate external expert priors, moving beyond the ``black-box" nature of fully automated amortized inference.

\begin{figure}[htbp]
    \centering
    \includegraphics[width=\linewidth]{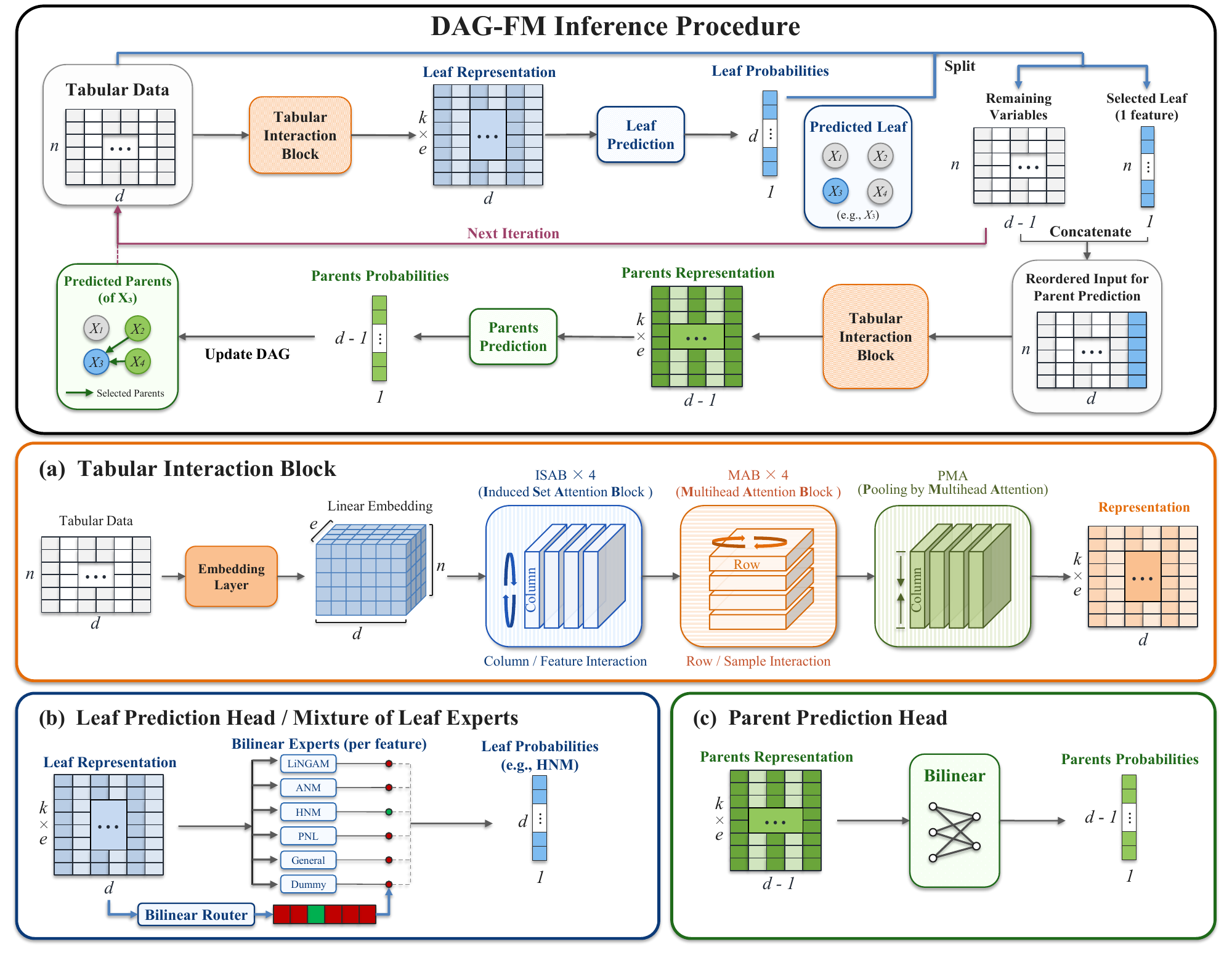}
    \vspace{-0.5em}
    \caption{Architecture and inference Procedure of DAG-FM. (a) Tabular Interaction Block; (b) Leaf Prediction Head; (c) Parent Prediction Head.}
    \label{fig:1}
\end{figure}

\subsection{Foundation Model Architecture}
DAG-FM employs two Transformer-based sub-modules, $Q_{\theta_1}(\boldsymbol{L}\!\mid\!\mathcal{D})$ and $Q_{\theta_2}(\boldsymbol{E}_{:,i}\!\mid\!L_i,\mathcal{D})$, to amortize the processes of leaf-node prediction and parent-node prediction, respectively. The architecture is illustrated in \cref{fig:1}.

\vspace{-0.5em}
\begin{wrapfigure}{r}{0.55\textwidth} 
    \vspace{-20pt} 
    \begin{minipage}{0.55\textwidth}
        \begin{algorithm}[H] 
            \caption{DAG-FM Inference Procedure}
            \label{alg:1}
            \begin{algorithmic}[1]
                \Require Tabular data $\mathcal{D}$, threshold $\alpha$
                \Ensure Causal DAG $\mathcal{G}$, causal ordering $\tau$
                \State $I \gets \{1, \dots, d\},\, \tau \gets [],\, \mathcal{E} \gets \emptyset$
                \While{$|I| > 0$}
                    \State $l \gets \arg\max_{i \in I} Q_{\theta_1}(L_i\!\mid\!\mathcal{D}_{I})$
                    \State $S \gets \{j \in I \setminus \{l\} \mid Q_{\theta_2}(E_{jl}\!\mid\!L_l, \mathcal{D}_{I}) > \alpha\}$
                    \For{$s \in S$}
                        \State $\mathcal{E} \gets \mathcal{E} \cup \{(s, l)\}$
                    \EndFor
                    \State $I \gets I \setminus \{l\},\, \tau \gets [\tau, l]$
                \EndWhile
                \State \Return $\mathcal{G}, \tau$
            \end{algorithmic}
        \end{algorithm}
    \end{minipage}
    \vspace{-10pt}
\end{wrapfigure}

\paragraph{Tabular Interaction Block}
Unlike previous approaches \citep{Lorch2022, Peng2026, Li2026} that utilize Axial Transformers to alternate between row and column interactions or rely on max-pooling for feature-wise representations, we leverage the backbone architecture from TabICLv2 \citep{Qu2026}. The block is composed of four cascaded modules: a linear embedding layer, a Set Transformer with four ISAB blocks \citep{Lee2019} for row-wise (sample) interaction, a Transformer with four MABs \citep{Vaswani2017} for column-wise (feature) interaction, and a PMA \citep{Lee2019} to aggregate row-specific information. As shown in \cref{fig:1} (a), this transformation yields a representation of size $k \times e$ for each feature, where $k$ denotes the number of seed vectors in the PMA and $e$ the embedding dimension.

\paragraph{Leaf Prediction Head}
We integrate the MoLE framework into the leaf prediction head. It consists of $m+1$ bilinear heads corresponding to the mechanism families in our prior space, outputting both routing probabilities and family-specific expert predictions. Following \cite{Lorch2022}, each bilinear head is parameterized as $\text{sigmoid}(\frac{1}{E}\sum_{i=1}^E\boldsymbol{u}_i^\intercal\boldsymbol{v}_i+\boldsymbol{b}_i)$, where $\boldsymbol{u}, \boldsymbol{v}$ are generated by linear and LayerNorm layers. We apply a top-$k$ strategy (setting $k=1$) to weight the expert predictions for the final leaf probability.

In practice, in addition to the four mechanism families specified in \cref{assum:1}, two supplementary experts are included: (i) An expert specifically dedicated to the aforementioned general mechanism; and (ii) As dummy expert designed to handle source nodes (since source nodes consist strictly of pure noise and can thus be classified into any mechanism family), as well as any other potential unseen cases or scenarios that cannot even be categorized under the general mechanism.

\paragraph{Parent Prediction Head}
We also employ a bilinear head followed by a $\text{sigmoid}$ activation to predict the probability of a node being a parent. However, to enable the model to distinguish between the queried entities (all potential parent nodes) and the query target (the leaf node) during tabular interactions, we position the known leaf node as the first feature input and incorporate an additional learnable type token \citep{Devlin2019} during the embedding stage. At inference time, nodes with a predicted probability exceeding the threshold $\alpha=0.5$ are identified as parent nodes.

\paragraph{Training Procedure}
The base models are optimized via the following amortized objectives:
\begin{align}
\arg\min_{\theta_1}-\mathbb{E}_{\mathcal{D},\mathcal{G}}\left[\log Q_{\theta_1}(\boldsymbol{L}\!\mid\!\mathcal{D})\right], \quad \arg\min_{\theta_2}-\mathbb{E}_{\mathcal{D},\mathcal{G},L_i}\left[\log Q_{\theta_2}(\boldsymbol{E}_{:,i}\!\mid\!L_i,\mathcal{D})\right],
\end{align}
where $\mathcal{D}, \mathcal{G}$ are sampled from the prior space, and $\boldsymbol{L}$ is the binary label vector of leaf nodes. Within MoLE, the leaf log-likelihood is factored as:
\begin{align}
\log Q_{\theta_1}(\boldsymbol{L}\!\mid\!\mathcal{D})=\log\sum_{\mathcal{F}_i}\left(Q_{\theta_1}(\boldsymbol{L}\!\mid\!\mathcal{F}_i,\mathcal{D}) \cdot Q_{\theta_1}(\mathcal{F}_i\!\mid\!\mathcal{D})\right),
\end{align}
where $Q_{\theta_1}(\mathcal{F}_i\!\mid\!\mathcal{D})$ denotes the routing probability. All objectives are optimized using the standard Binary Cross-Entropy (BCE) loss function.

\paragraph{Inference Procedure}
The DAG identification procedure is detailed in \cref{alg:1}.
\section{Experiments}

\paragraph{Baselines}
Our experiments encompass the following causal discovery baselines: (i) FCM-based methods: Direct-LiNGAM \citep{Shimizu2011}, ICA-LiNGAM \citep{Shimizu2006}, CAM \citep{Buhlmann2014}, RESIT \citep{Peters2014}, SCORE \citep{Rolland2022}, DAS \citep{Montagna2023b}, NoGAM \citep{Montagna2023}, HOST \citep{Duong2023}, CaPS \citep{Xu2024}, and SkewScore \citep{Lin2025}; (ii) Amortized algorithms: AVICI \citep{Lorch2022}, CauScale \citep{Peng2026}, TabCausal \citep{Li2026}, and FoundCause \citep{Blobaum2026}.

\paragraph{Metrics}
For the DAG identification task, we report edge-level precision (\texttt{Prec.}), recall (\texttt{Rec.}), F1 score (\texttt{F1}), Jaccard index (\texttt{Jac.}), and normalized Structural Hamming Distance (\texttt{nSHD}). To emphasize the capability of discovering DAG structures rather than merely predicting adjacency matrices, we add an extra penalty of 1 to the SHD if the predicted graph contains reversed edges, and an additional 1 if it contains cycles. For the causal order identification task, we report the Order Divergence Rate (\texttt{ODR}) \citep{Rolland2022}.

\paragraph{Synthetic Benchmark Experiments}
First, we construct a synthetic benchmark, \textbf{Hetero}, by sampling 100 distinct SCMs from the prior space (see \cref{app:b_bench} for detailed construction) to evaluate the DAG identification performance of the aforementioned methods under heterogeneous causal mechanisms. The results are presented in \cref{tab:1}. On the DAG identification task, DAG-FM significantly outperforms several concurrently proposed causal discovery foundation models, achieving the best performance across all metrics.

\begin{table}[htbp]
\centering
\resizebox{\linewidth}{!}{%
\begin{tabular}{l | c c c c c}
\toprule
\multicolumn{6}{c}{\textbf{DAG Performance on Synthetic Heterogeneous Mechanism Benchmark ($n=1000, d=20$)}} \\
\midrule
Method & Prec. $\uparrow$ & Rec. $\uparrow$ & F1 $\uparrow$ & Jac. $\uparrow$ & nSHD $\downarrow$ \\
\midrule
AVICI \citep{Lorch2022} & $0.58_{0.19}$ & $0.30_{0.19}$ & $0.37_{0.20}$ & $0.25_{0.16}$ & $0.23_{0.21}$ \\
CauScale \citep{Peng2026} & $0.57_{0.18}$ & $0.34_{0.19}$ & $0.41_{0.19}$ & $0.27_{0.16}$ & $0.23_{0.21}$ \\
TabCausal \citep{Li2026} & $\underline{0.63}_{0.18}$ & $\underline{0.50}_{0.19}$ & $0.54_{0.17}$ & $0.39_{0.17}$ & $\underline{0.21}_{0.21}$ \\
FoundCause \citep{Blobaum2026} & $0.57_{0.22}$ & $\textbf{0.66}_{0.18}$ & $\underline{0.57}_{0.16}$ & $\underline{0.42}_{0.16}$ & $0.23_{0.17}$ \\
\textbf{DAG-FM (ours)} & $\textbf{0.71}_{0.14}$ & $\textbf{0.66}_{0.19}$ & $\textbf{0.67}_{0.15}$ & $\textbf{0.52}_{0.16}$ & $\textbf{0.17}_{0.17}$ \\
\bottomrule
\end{tabular}%
}
\caption{DAG identification performance on the synthetic heterogeneous mechanism benchmark (higher is better for \texttt{Prec.}, \texttt{Rec.}, \texttt{F1}, and \texttt{Jac.}; lower is better for \texttt{nSHD}), compared with other amortized algorithms. Subscripts denote 95\% confidence intervals.}
\label{tab:1}
\end{table}

Next, we construct additional synthetic benchmarks in the same manner from subspaces of specific causal mechanism families (see \cref{app:b_bench} for details) to demonstrate the accuracy of DAG-FM in identifying causal orders under homogeneous causal mechanisms. As shown in \cref{tab:2}, on the causal order identification task, DAG-FM is second only to Direct-LiNGAM on LiNGAM, while significantly surpassing other FCM-based algorithms across all other settings.

\begin{table}[htbp]
\centering
\resizebox{\linewidth}{!}{%
\begin{tabular}{l | c | c | c | c | c | c}
\toprule
\multicolumn{7}{c}{\textbf{{Order Performance on All Synthetic Benchmarks ($n=1000, d=20$)}}} \\
\midrule
Method & \multicolumn{1}{c|}{\textbf{LiNGAM}} & \multicolumn{1}{c|}{\textbf{ANM}} & \multicolumn{1}{c|}{\textbf{HNM}} & \multicolumn{1}{c|}{\textbf{PNL}} & \multicolumn{1}{c|}{\textbf{General}} & \multicolumn{1}{c}{\textbf{Hetero}} \\
\midrule
ICA-LiNGAM \citep{Shimizu2011} & $0.18_{0.15}$ & $0.50_{0.13}$ & $0.49_{0.16}$ & $0.47_{0.12}$ & $0.49_{0.12}$ & $0.48_{0.12}$ \\
Direct-LiNGAM \citep{Shimizu2006} & $\textbf{0.02}_{0.03}$ & $0.41_{0.13}$ & $0.41_{0.15}$ & $0.40_{0.14}$ & $0.52_{0.13}$ & $0.38_{0.15}$ \\
CAM \citep{Buhlmann2014} & $0.63_{0.18}$ & $0.30_{0.12}$ & $0.47_{0.18}$ & $0.42_{0.12}$ & $\underline{0.27}_{0.13}$ & $0.40_{0.13}$ \\
RESIT \citep{Peters2014} & $0.26_{0.15}$ & $\underline{0.21}_{0.13}$ & $0.44_{0.15}$ & $\underline{0.35}_{0.13}$ & $0.41_{0.13}$ & $0.37_{0.12}$ \\
SCORE \citep{Rolland2022} & $0.36_{0.14}$ & $0.29_{0.14}$ & $0.33_{0.16}$ & $0.38_{0.14}$ & $0.35_{0.16}$ & $0.35_{0.15}$ \\
DAS \citep{Montagna2023b} & $0.36_{0.14}$ & $0.29_{0.14}$ & $0.33_{0.16}$ & $0.38_{0.14}$ & $0.35_{0.16}$ & $0.35_{0.15}$ \\
NoGAM \citep{Montagna2023} & $0.33_{0.13}$ & $0.28_{0.14}$ & $\underline{0.32}_{0.15}$ & $0.37_{0.14}$ & $0.33_{0.15}$ & $\underline{0.34}_{0.14}$ \\
HOST \citep{Duong2023} & $0.66_{0.15}$ & $0.58_{0.14}$ & $0.55_{0.15}$ & $0.51_{0.13}$ & $0.60_{0.13}$ & $0.57_{0.14}$ \\
CaPS \citep{Xu2024} & $0.35_{0.15}$ & $0.30_{0.14}$ & $\underline{0.32}_{0.16}$ & $0.37_{0.14}$ & $0.35_{0.16}$ & $0.35_{0.15}$ \\
SkewScore \citep{Lin2025} & $0.54_{0.14}$ & $0.55_{0.16}$ & $0.53_{0.17}$ & $0.54_{0.13}$ & $0.52_{0.14}$ & $0.50_{0.15}$ \\
\textbf{DAG-FM (ours)} & $\underline{0.09}_{0.08}$ & $\textbf{0.16}_{0.10}$ & $\textbf{0.17}_{0.11}$ & $\textbf{0.14}_{0.10}$ & $\textbf{0.21}_{0.11}$ & $\textbf{0.16}_{0.10}$ \\
\bottomrule
\end{tabular}%
}
\caption{Causal order identification performance across all synthetic benchmarks (lower is better for \texttt{ODR}), compared with other FCM-based methods. Here, the benchmarks \textbf{LiNGAM}, \textbf{ANM}, \textbf{HNM}, and \textbf{PNL} represent homogeneous causal mechanisms, \textbf{General} denotes the general causal mechanisms in \cref{sec:4_prior}, and \textbf{Hetero} corresponds to the benchmark in \cref{tab:1}. Subscripts denote 95\% confidence intervals.}
\label{tab:2}
\end{table}

\paragraph{Real-World Benchmark Experiments}
Finally, we apply DAG-FM to real-world benchmarks to evaluate its generalization capability. We select two benchmarks: (i) Sachs \citep{Sachs2005}: A protein signaling dataset containing 7,466 samples and 11 variables; (ii) Causal Chamber \citep{Gamella2025}: A diverse set of physical causal systems containing 10,000 samples across 38 variables, from which we use only the 20 non-constant variables.

\begin{table*}[htbp]
\centering
\resizebox{\linewidth}{!}{%
\begin{tabular}{l | c c c c c | c c c c c}
\toprule
\multicolumn{11}{c}{\textbf{{DAG Performance on Real-world Benchmarks}}} \\
\midrule
\multirow{2}{*}{Method} & \multicolumn{5}{c|}{\textbf{{Sachs}} ($n=7466, d=11$)} & \multicolumn{5}{c}{\textbf{{Causal Chamber}} ($n=10\text{{k}}, d=20$)} \\
 & Prec. $\uparrow$ & Rec. $\uparrow$ & F1 $\uparrow$ & Jac. $\uparrow$ & nSHD $\downarrow$ & Prec. $\uparrow$ & Rec. $\uparrow$ & F1 $\uparrow$ & Jac. $\uparrow$ & nSHD $\downarrow$ \\
\midrule
AVICI \citep{Lorch2022} & \multicolumn{5}{c|}{Out of Memory} & \multicolumn{5}{c}{Out of Memory} \\
CauScale \citep{Peng2026} & \multicolumn{5}{c|}{Out of Memory} & \multicolumn{5}{c}{Out of Memory} \\
TabCausal \citep{Li2026} & 0.38 & \underline{0.30} & \underline{0.33} & \underline{0.20} & 0.44 & \underline{0.71} & \underline{0.51} & \underline{0.60} & \underline{0.43} & \underline{0.14} \\
FoundCause \citep{Blobaum2026} & \underline{0.50} & 0.25 & \underline{0.33} & \underline{0.20} & \underline{0.36} & 0.53 & \textbf{0.59} & 0.56 & 0.39 & 0.19 \\
\textbf{DAG-FM (ours)} & \textbf{0.54} & \textbf{0.35} & \textbf{0.42} & \textbf{0.27} & \textbf{0.35} & \textbf{0.77} & \underline{0.51} & \textbf{0.62} & \textbf{0.44} & \textbf{0.13} \\
\bottomrule
\end{tabular}%
}
\caption{DAG identification performance on two real-world benchmarks (higher is better for \texttt{Prec.}, \texttt{Rec.}, \texttt{F1}, and \texttt{Jac.}; lower is better for \texttt{nSHD}), compared with other amortized algorithms.}
\label{tab:3}
\end{table*}

As shown in \cref{tab:3}, DAG-FM is capable of scaling to these datasets (whereas several baseline algorithms fail to directly handle high-dimensional, large-sample regimes) and achieves state-of-the-art performance across multiple metrics, outperforming both traditional methods and causal discovery foundation models. This demonstrates that DAG-FM generalizes well to real-world problems and provides solid support for its practical deployment.

\paragraph{Additional Results}
In \cref{app:c}, we provide additional experimental results under larger sample sizes, higher dimensions, various perturbations, and out-of-distribution (OOD) settings to further demonstrate the robustness and generalization of DAG-FM. Furthermore, \cref{app:d} presents ablation studies on the MoLE framework to highlight its effectiveness, along with more comprehensive experimental evaluations covering a broader range of baseline models.
\section{Conclusion}

In this work, we presented \textbf{DAG-FM}, a highly adaptable and scalable foundation model for causal discovery from observational tabular data. By moving away from direct, monolithic adjacency matrix prediction, DAG-FM reframes causal DAG identification as a sequential prediction process. It employs two amortized Transformer sub-modules to iteratively identify leaf nodes and their respective parents, ensuring the structural validity of the generated causal graph. A key contribution of our architecture is the integration of the MoLE framework, which effectively bridges the gap between theoretical FCM assumptions and practical data by dynamically routing predictions across mixed, diverse, heterogeneous causal mechanisms.

Extensive empirical evaluations confirm the superiority of DAG-FM. On synthetic benchmarks featuring heterogeneous causal mechanisms, DAG-FM outperforms existing amortized baselines, including powerful concurrent causal discovery foundation models, in DAG recovery. Moreover, DAG-FM significantly surpasses traditional FCM-based methods in causal order identification across both heterogeneous and homogeneous causal mechanisms. Furthermore, its state-of-the-art performance on two real-world benchmarks underscores its robust generalizability and capacity to handle high-dimensional, large-sample settings where traditional algorithms often struggle.

\paragraph{Future Work} While DAG-FM demonstrates formidable capabilities in causal discovery under the assumption of causal sufficiency, extending this framework to handle unobserved latent confounders and missing data presents an exciting avenue for future research. Additionally, scaling the foundation model with broader and more diverse pre-training corpora could further enhance its zero-shot inference capabilities across interdisciplinary scientific domains. Finally, as DAG-FM is currently tailored for observational causal discovery, extending its theoretical identifiability guarantees and framework to interventional settings represents another promising direction.
\bibliography{conference_bibtex}

@book{Pearl2009,
  title        = {Causality: Models, Reasoning and Inference},
  author       = {Pearl, Judea},
  year         = 2009,
  publisher    = {Cambridge University Press},
  address      = {USA},
  pages        = 207,
  isbn         = {052189560X},
  edition      = {2nd}
}

@article{Spirtes1991,
  title        = {An algorithm for fast recovery of sparse causal graphs},
  author       = {Spirtes, Peter and Glymour, Clark},
  year         = 1991,
  journal      = {Social science computer review},
  publisher    = {Sage Publications Sage CA: Thousand Oaks, CA},
  volume       = 9,
  number       = 1,
  pages        = {62--72}
}

@inproceedings{Spirtes1995,
  title        = {Directed cyclic graphical representations of feedback models},
  author       = {Spirtes, Peter},
  year         = 1995,
  booktitle    = {Proceedings of the Eleventh Conference on Uncertainty in Artificial Intelligence},
  location     = {Montr\'{e}al, Qu\'{e}, Canada},
  publisher    = {Morgan Kaufmann Publishers Inc.},
  address      = {San Francisco, CA, USA},
  series       = {UAI'95},
  pages        = {491–498},
  isbn         = 1558603859,
  numpages     = 8
}

@article{Cooper1992,
  title        = {A Bayesian method for the induction of probabilistic networks from data},
  author       = {Cooper, Gregory F and Herskovits, Edward},
  year         = 1992,
  month        = oct,
  journal      = {Mach. Learn.},
  publisher    = {Springer Science and Business Media LLC},
  volume       = 9,
  number       = 4,
  pages        = {309--347},
  language     = {en}
}

@article{Chickering2002,
  title        = {Optimal structure identification with greedy search},
  author       = {Chickering, David Maxwell},
  year         = 2002,
  journal      = {Journal of machine learning research},
  volume       = 3,
  number       = {Nov},
  pages        = {507--554}
}

@article{Zheng2018,
  title        = {Dags with no tears: Continuous optimization for structure learning},
  author       = {Zheng, Xun and Aragam, Bryon and Ravikumar, Pradeep K and Xing, Eric P},
  year         = 2018,
  journal      = {Advances in neural information processing systems},
  volume       = 31
}

@article{Bello2022,
  title        = {Dagma: Learning dags via m-matrices and a log-determinant acyclicity characterization},
  author       = {Bello, Kevin and Aragam, Bryon and Ravikumar, Pradeep},
  year         = 2022,
  journal      = {Advances in Neural Information Processing Systems},
  volume       = 35,
  pages        = {8226--8239}
}

@inproceedings{Hoyer2008,
  title        = {Nonlinear causal discovery with additive noise models},
  author       = {Hoyer, Patrik and Janzing, Dominik and Mooij, Joris M and Peters, Jonas and Sch\"{o}lkopf, Bernhard},
  year         = 2008,
  booktitle    = {Advances in Neural Information Processing Systems},
  publisher    = {Curran Associates, Inc.},
  volume       = 21,
  pages        = {},
  editor       = {D. Koller and D. Schuurmans and Y. Bengio and L. Bottou}
}

@inproceedings{Tagasovska2020,
  title        = {Distinguishing Cause from Effect Using Quantiles: Bivariate Quantile Causal Discovery},
  author       = {Tagasovska, Natasa and Chavez-Demoulin, Val{\'e}rie and Vatter, Thibault},
  year         = 2020,
  month        = {13--18 Jul},
  booktitle    = {Proceedings of the 37th International Conference on Machine Learning},
  publisher    = {PMLR},
  series       = {Proceedings of Machine Learning Research},
  volume       = 119,
  pages        = {9311--9323},
  editor       = {III, Hal Daumé and Singh, Aarti},
}

@inproceedings{Zhang2009,
  title        = {On the identifiability of the post-nonlinear causal model},
  author       = {Zhang, Kun and Hyv\"{a}rinen, Aapo},
  year         = 2009,
  booktitle    = {Proceedings of the Twenty-Fifth Conference on Uncertainty in Artificial Intelligence},
  location     = {Montreal, Quebec, Canada},
  publisher    = {AUAI Press},
  address      = {Arlington, Virginia, USA},
  series       = {UAI '09},
  pages        = {647–655},
  isbn         = 9780974903958,
  numpages     = 9
}

@inproceedings{Xi2025,
  title        = {Distinguishing Cause from Effect with Causal Velocity Models},
  author       = {Johnny Xi and Hugh Dance and Peter Orbanz and Benjamin Bloem-Reddy},
  year         = 2025,
  booktitle    = {Forty-second International Conference on Machine Learning}
}

@inproceedings{Immer2023,
  title        = {On the Identifiability and Estimation of Causal Location-Scale Noise Models},
  author       = {Immer, Alexander and Schultheiss, Christoph and Vogt, Julia E and Sch\"{o}lkopf, Bernhard and B\"{u}hlmann, Peter and Marx, Alexander},
  year         = 2023,
  month        = {23--29 Jul},
  booktitle    = {Proceedings of the 40th International Conference on Machine Learning},
  publisher    = {PMLR},
  series       = {Proceedings of Machine Learning Research},
  volume       = 202,
  pages        = {14316--14332},
  editor       = {Krause, Andreas and Brunskill, Emma and Cho, Kyunghyun and Engelhardt, Barbara and Sabato, Sivan and Scarlett, Jonathan}
}

@article{Strobl2023,
  title        = {Identifying patient-specific root causes with the heteroscedastic noise model},
  author       = {Eric V. Strobl and Thomas A. Lasko},
  year         = 2023,
  journal      = {Journal of Computational Science},
  volume       = 72,
  pages        = 102099,
  issn         = {1877-7503}
}

@inproceedings{Duong2022,
  title        = {Bivariate Causal Discovery via Conditional Divergence},
  author       = {Duong, Bao and Nguyen, Thin},
  year         = 2022,
  month        = {11--13 Apr},
  booktitle    = {Proceedings of the First Conference on Causal Learning and Reasoning},
  publisher    = {PMLR},
  series       = {Proceedings of Machine Learning Research},
  volume       = 177,
  pages        = {236--252},
  editor       = {Schölkopf, Bernhard and Uhler, Caroline and Zhang, Kun},
}

@article{Shimizu2006,
  title        = {A Linear Non-Gaussian Acyclic Model for Causal Discovery},
  author       = {Shohei Shimizu and Patrik O. Hoyer and Aapo Hyv{\"a}rinen and Antti Kerminen},
  year         = 2006,
  journal      = {Journal of Machine Learning Research},
  volume       = 7,
  number       = 72,
  pages        = {2003--2030}
}

@inproceedings{Ormaniec2025,
  title        = {Standardizing Structural Causal Models},
  author       = {Weronika Ormaniec and Scott Sussex and Lars Lorch and Bernhard Sch{\"o}lkopf and Andreas Krause},
  year         = 2025,
  booktitle    = {The Thirteenth International Conference on Learning Representations}
}

@inproceedings{Reisach2021,
  title        = {Beware of the Simulated DAG! Causal Discovery Benchmarks May Be Easy to Game},
  author       = {Reisach, Alexander and Seiler, Christof and Weichwald, Sebastian},
  year         = 2021,
  booktitle    = {Advances in Neural Information Processing Systems},
  publisher    = {Curran Associates, Inc.},
  volume       = 34,
  pages        = {27772--27784},
  editor       = {M. Ranzato and A. Beygelzimer and Y. Dauphin and P.S. Liang and J. Wortman Vaughan}
}

@inproceedings{Reisach2023,
  title        = {A Scale-Invariant Sorting Criterion to Find a Causal Order in Additive Noise Models},
  author       = {Reisach, Alexander and Tami, Myriam and Seiler, Christof and Chambaz, Antoine and Weichwald, Sebastian},
  year         = 2023,
  booktitle    = {Advances in Neural Information Processing Systems},
  publisher    = {Curran Associates, Inc.},
  volume       = 36,
  pages        = {785--807},
  editor       = {A. Oh and T. Naumann and A. Globerson and K. Saenko and M. Hardt and S. Levine}
}

@article{Shimizu2011,
  title        = {DirectLiNGAM: A Direct Method for Learning a Linear Non-Gaussian Structural Equation Model},
  author       = {Shohei Shimizu and Takanori Inazumi and Yasuhiro Sogawa and Aapo Hyv{{\"a}}rinen and Yoshinobu Kawahara and Takashi Washio and Patrik O. Hoyer and Kenneth Bollen},
  year         = 2011,
  journal      = {Journal of Machine Learning Research},
  volume       = 12,
  number       = 33,
  pages        = {1225--1248}
}

@incollection{Duong2023,
  title        = {Heteroscedastic Causal Structure Learning},
  author       = {Duong, Bao and Nguyen, Thin},
  year         = 2023,
  month        = sep,
  booktitle    = {Frontiers in Artificial Intelligence and Applications},
  publisher    = {IOS Press},
  series       = {Frontiers in artificial intelligence and applications}
}

@inproceedings{Lin2025,
  title        = {A Skewness-Based Criterion for Addressing Heteroscedastic Noise in Causal Discovery},
  author       = {Yingyu Lin and Yuxing Huang and Wenqin Liu and Haoran Deng and Ignavier Ng and Kun Zhang and Mingming Gong and Yian Ma and Biwei Huang},
  year         = 2025,
  booktitle    = {The Thirteenth International Conference on Learning Representations}
}

@inproceedings{Rolland2022,
  title        = {Score Matching Enables Causal Discovery of Nonlinear Additive Noise Models},
  author       = {Rolland, Paul and Cevher, Volkan and Kleindessner, Matth{\"a}us and Russell, Chris and Janzing, Dominik and Sch{\"o}lkopf, Bernhard and Locatello, Francesco},
  year         = 2022,
  month        = {17--23 Jul},
  booktitle    = {Proceedings of the 39th International Conference on Machine Learning},
  publisher    = {PMLR},
  series       = {Proceedings of Machine Learning Research},
  volume       = 162,
  pages        = {18741--18753},
  editor       = {Chaudhuri, Kamalika and Jegelka, Stefanie and Song, Le and Szepesvari, Csaba and Niu, Gang and Sabato, Sivan},
}

@inproceedings{Montagna2023,
  title        = {Causal Discovery with Score Matching on Additive Models with Arbitrary Noise},
  author       = {Montagna, Francesco and Noceti, Nicoletta and Rosasco, Lorenzo and Zhang, Kun and Locatello, Francesco},
  year         = 2023,
  month        = {11--14 Apr},
  booktitle    = {Proceedings of the Second Conference on Causal Learning and Reasoning},
  publisher    = {PMLR},
  series       = {Proceedings of Machine Learning Research},
  volume       = 213,
  pages        = {726--751},
  editor       = {van der Schaar, Mihaela and Zhang, Cheng and Janzing, Dominik},
}

@inproceedings{Xu2024,
  title        = {Ordering-Based Causal Discovery for Linear and Nonlinear Relations},
  author       = {Xu, Zhuopeng and Li, Yujie and Liu, Cheng and Gui, Ning},
  year         = 2024,
  booktitle    = {Advances in Neural Information Processing Systems},
  publisher    = {Curran Associates, Inc.},
  volume       = 37,
  pages        = {4315--4340},
  editor       = {A. Globerson and L. Mackey and D. Belgrave and A. Fan and U. Paquet and J. Tomczak and C. Zhang}
}

@article{Sachs2005,
  title        = {Causal Protein-Signaling Networks Derived from Multiparameter Single-Cell Data},
  author       = {Karen Sachs  and Omar Perez  and Dana Pe'er  and Douglas A. Lauffenburger  and Garry P. Nolan},
  year         = 2005,
  journal      = {Science},
  volume       = 308,
  number       = 5721,
  pages        = {523--529},
  doi          = {10.1126/science.1105809},
}

@article{Peters2014,
  title        = {Causal Discovery with Continuous Additive Noise Models},
  author       = {Jonas Peters and Joris M. Mooij and Dominik Janzing and Bernhard Sch{{\"o}}lkopf},
  year         = 2014,
  journal      = {Journal of Machine Learning Research},
  volume       = 15,
  number       = 58,
  pages        = {2009--2053}
}

@article{Buhlmann2014,
  title        = {CAM: CAUSAL ADDITIVE MODELS, HIGH-DIMENSIONAL ORDER SEARCH AND PENALIZED REGRESSION},
  author       = {Peter Bühlmann and Jonas Peters and Jan Ernest},
  year         = 2014,
  journal      = {The Annals of Statistics},
  publisher    = {Institute of Mathematical Statistics},
  volume       = 42,
  number       = 6,
  pages        = {2526--2556},
  issn         = {00905364}
}

@inproceedings{Yu2019,
  title        = {{DAG}-{GNN}: {DAG} Structure Learning with Graph Neural Networks},
  author       = {Yu, Yue and Chen, Jie and Gao, Tian and Yu, Mo},
  year         = 2019,
  month        = {09--15 Jun},
  booktitle    = {Proceedings of the 36th International Conference on Machine Learning},
  publisher    = {PMLR},
  series       = {Proceedings of Machine Learning Research},
  volume       = 97,
  pages        = {7154--7163},
  editor       = {Chaudhuri, Kamalika and Salakhutdinov, Ruslan}
}

@article{Zhang2013,
title = {Integrated Systems Approach Identifies Genetic Nodes and Networks in Late-Onset Alzheimer’s Disease},
journal = {Cell},
volume = {153},
number = {3},
pages = {707-720},
year = {2013},
issn = {0092-8674},
doi = {https://doi.org/10.1016/j.cell.2013.03.030},
author = {Bin Zhang and Chris Gaiteri and Liviu-Gabriel Bodea and Zhi Wang and Joshua McElwee and Alexei A. Podtelezhnikov and Chunsheng Zhang and Tao Xie and Linh Tran and Radu Dobrin and Eugene Fluder and Bruce Clurman and Stacey Melquist and Manikandan Narayanan and Christine Suver and Hardik Shah and Milind Mahajan and Tammy Gillis and Jayalakshmi Mysore and Marcy E. MacDonald and John R. Lamb and David A. Bennett and Cliona Molony and David J. Stone and Vilmundur Gudnason and Amanda J. Myers and Eric E. Schadt and Harald Neumann and Jun Zhu and Valur Emilsson},
}

@article{Vandenbroucke2016,
    author = {Vandenbroucke, Jan P and Broadbent, Alex and Pearce, Neil},
    title = {Causality and causal inference in epidemiology: the need for a pluralistic approach},
    journal = {International Journal of Epidemiology},
    volume = {45},
    number = {6},
    pages = {1776-1786},
    year = {2016},
    month = {12},
    issn = {0300-5771},
    doi = {10.1093/ije/dyv341},
    eprint = {https://academic.oup.com/ije/article-pdf/45/6/1776/24170420/dyv341.pdf},
}

@ARTICLE{Huber2024,
  title     = "An introduction to causal discovery",
  author    = "Huber, Martin",
  journal   = "Schweiz. Z. Volkswirtsch. Stat.",
  publisher = "Springer Science and Business Media LLC",
  volume    =  160,
  number    =  1,
  month     =  oct,
  year      =  2024,
  copyright = "https://creativecommons.org/licenses/by/4.0",
  language  = "en"
}

@Article{Vukovic2022,
AUTHOR = {Vukovi\'{c}, Matej and Thalmann, Stefan},
TITLE = {Causal Discovery in Manufacturing: A Structured Literature Review},
JOURNAL = {Journal of Manufacturing and Materials Processing},
VOLUME = {6},
YEAR = {2022},
NUMBER = {1},
ARTICLE-NUMBER = {10},
ISSN = {2504-4494},
DOI = {10.3390/jmmp6010010}
}

@inproceedings{Lorch2022,
 author = {Lorch, Lars and Sussex, Scott and Rothfuss, Jonas and Krause, Andreas and Sch\"{o}lkopf, Bernhard},
 booktitle = {Advances in Neural Information Processing Systems},
 editor = {S. Koyejo and S. Mohamed and A. Agarwal and D. Belgrave and K. Cho and A. Oh},
 pages = {13104--13118},
 publisher = {Curran Associates, Inc.},
 title = {Amortized Inference for Causal Structure Learning},
 volume = {35},
 year = {2022}
}

@inproceedings{Dhir2025,
title={A Meta-Learning Approach to Bayesian Causal Discovery},
author={Anish Dhir and Matthew Ashman and James Requeima and Mark van der Wilk},
booktitle={The Thirteenth International Conference on Learning Representations},
year={2025},
}

@misc{Thompson2026,
      title={Arrow: A Foundation Model for Causal Discovery}, 
      author={Ryan Thompson and He Zhao and Daniel M. Steinberg and Edwin V. Bonilla},
      year={2026},
      eprint={2605.07204},
      archivePrefix={arXiv},
      primaryClass={cs.LG},
}

@inproceedings{
Peng2026,
title={CauScale: Neural Causal Discovery at Scale},
author={Bo Peng and Sirui Chen and Jiaguo Tian and Yu Qiao and Chaochao Lu},
booktitle={Forty-third International Conference on Machine Learning},
year={2026},
}

@misc{Li2026,
      title={TabCausal: Pretraining Across Causal Environments for Tabular Causal Discovery}, 
      author={Zi-Rong Li and Si-Yang Liu and Tian-Zuo Wang and Han-Jia Ye},
      year={2026},
      eprint={2605.31156},
      archivePrefix={arXiv},
      primaryClass={cs.LG},
}

@misc{Montagna2025,
title={Demystifying amortized causal discovery with transformers},
author={Francesco Montagna and Max Cairney-Leeming and Dhanya Sridhar and Francesco Locatello},
year={2025},
}

@inproceedings{
Chen2026,
title={Causal Discovery via Quantile Partial Effect},
author={Yikang Chen and Xingzhe Sun and Dehui du},
booktitle={The Fourteenth International Conference on Learning Representations},
year={2026},
}

@InProceedings{Montagna2023b,
  title = 	 {Scalable Causal Discovery with Score Matching},
  author =       {Montagna, Francesco and Noceti, Nicoletta and Rosasco, Lorenzo and Zhang, Kun and Locatello, Francesco},
  booktitle = 	 {Proceedings of the Second Conference on Causal Learning and Reasoning},
  pages = 	 {752--771},
  year = 	 {2023},
  editor = 	 {van der Schaar, Mihaela and Zhang, Cheng and Janzing, Dominik},
  volume = 	 {213},
  series = 	 {Proceedings of Machine Learning Research},
  month = 	 {11--14 Apr},
  publisher =    {PMLR},
}

@inproceedings{
Ke2023,
title={Learning to Induce Causal Structure },
author={Nan Rosemary Ke and Silvia Chiappa and Jane X Wang and Jorg Bornschein and Anirudh Goyal and Melanie Rey and Theophane Weber and Matthew Botvinick and Michael Curtis Mozer and Danilo Jimenez Rezende},
booktitle={International Conference on Learning Representations},
year={2023},
}

@inproceedings{
Zhang2025,
title={Learning Identifiable Structures Avoids Bias in {DNN}-based Supervised Causal Learning},
author={Jiaru Zhang and Rui Ding and Qiang Fu and Huang Bojun and zizhen Deng and Yang Hua and Haibing Guan and Shi Han and Dongmei Zhang},
booktitle={The 28th International Conference on Artificial Intelligence and Statistics},
year={2025},
}

@article{
Kim2025,
title={Large-Scale Targeted Cause Discovery via Learning from Simulated Data},
author={Jang-Hyun Kim and Claudia Skok Gibbs and Sangdoo Yun and Hyun Oh Song and Kyunghyun Cho},
journal={Transactions on Machine Learning Research},
issn={2835-8856},
year={2025},
note={}
}

@misc{Xu2026,
      title={Learning Causal Orderings for In-Context Tabular Prediction}, 
      author={Sascha Xu and Sarah Mameche and Jilles Vreeken},
      year={2026},
      eprint={2605.22335},
      archivePrefix={arXiv},
      primaryClass={cs.LG},
}

@ARTICLE{Hollmann2025,
  title     = "Accurate predictions on small data with a tabular foundation
               model",
  author    = "Hollmann, Noah and M{\"u}ller, Samuel and Purucker, Lennart and
               Krishnakumar, Arjun and K{\"o}rfer, Max and Hoo, Shi Bin and
               Schirrmeister, Robin Tibor and Hutter, Frank",
  journal   = "Nature",
  publisher = "Springer Science and Business Media LLC",
  volume    =  637,
  number    =  8045,
  pages     = "319--326",
  month     =  jan,
  year      =  2025,
  copyright = "https://creativecommons.org/licenses/by/4.0",
  language  = "en"
}

@article{Qu2026,
  title={{TabICLv2}: {A} better, faster, scalable, and open tabular foundation model},
  author={Qu, Jingang and Holzm{\"u}ller, David and Varoquaux, Ga{\"e}l and Le Morvan, Marine},
  booktitle={International Conference on Machine Learning},
  year={2026}
}

@InProceedings{Lee2019,
  title = 	 {Set Transformer: A Framework for Attention-based Permutation-Invariant Neural Networks},
  author =       {Lee, Juho and Lee, Yoonho and Kim, Jungtaek and Kosiorek, Adam and Choi, Seungjin and Teh, Yee Whye},
  booktitle = 	 {Proceedings of the 36th International Conference on Machine Learning},
  pages = 	 {3744--3753},
  year = 	 {2019},
  editor = 	 {Chaudhuri, Kamalika and Salakhutdinov, Ruslan},
  volume = 	 {97},
  series = 	 {Proceedings of Machine Learning Research},
  month = 	 {09--15 Jun},
  publisher =    {PMLR},
}

@inproceedings{Vaswani2017,
 author = {Vaswani, Ashish and Shazeer, Noam and Parmar, Niki and Uszkoreit, Jakob and Jones, Llion and Gomez, Aidan N and Kaiser, \L ukasz and Polosukhin, Illia},
 booktitle = {Advances in Neural Information Processing Systems},
 editor = {I. Guyon and U. Von Luxburg and S. Bengio and H. Wallach and R. Fergus and S. Vishwanathan and R. Garnett},
 pages = {},
 publisher = {Curran Associates, Inc.},
 title = {Attention is All you Need},
 volume = {30},
 year = {2017}
}

@InProceedings{Maeda2020,
  title = 	 {RCD: Repetitive causal discovery of linear non-Gaussian acyclic models with latent confounders},
  author =       {Maeda, Takashi Nicholas and Shimizu, Shohei},
  booktitle = 	 {Proceedings of the Twenty Third International Conference on Artificial Intelligence and Statistics},
  pages = 	 {735--745},
  year = 	 {2020},
  editor = 	 {Chiappa, Silvia and Calandra, Roberto},
  volume = 	 {108},
  series = 	 {Proceedings of Machine Learning Research},
  month = 	 {26--28 Aug},
  publisher =    {PMLR},
}

@inproceedings{
Lachapelle2020,
title={Gradient-Based Neural DAG Learning},
author={Sébastien Lachapelle and Philippe Brouillard and Tristan Deleu and Simon Lacoste-Julien},
booktitle={International Conference on Learning Representations},
year={2020},
}

@inproceedings{Ng2020,
 author = {Ng, Ignavier and Ghassami, AmirEmad and Zhang, Kun},
 booktitle = {Advances in Neural Information Processing Systems},
 editor = {H. Larochelle and M. Ranzato and R. Hadsell and M.F. Balcan and H. Lin},
 pages = {17943--17954},
 publisher = {Curran Associates, Inc.},
 title = {On the Role of Sparsity and DAG Constraints for Learning Linear DAGs},
 volume = {33},
 year = {2020}
}

@ARTICLE{Gamella2025,
  title     = "Causal chambers as a real-world physical testbed for {AI}
               methodology",
  author    = "Gamella, Juan L and Peters, Jonas and B{\"u}hlmann, Peter",
  journal   = "Nat. Mach. Intell.",
  publisher = "Springer Science and Business Media LLC",
  volume    =  7,
  number    =  1,
  pages     = "107--118",
  month     =  jan,
  year      =  2025,
  copyright = "https://creativecommons.org/licenses/by/4.0",
  language  = "en"
}

@misc{Blobaum2026,
      title={FoundCause: Causal Discovery with Latent Confounders from Observational Data}, 
      author={Patrick Bl{{\"o}}baum and Krishnakumar Balasubramanian and Shiva Prasad Kasiviswanathan},
      year={2026},
      eprint={2606.17516},
      archivePrefix={arXiv},
      primaryClass={cs.LG},
}

@misc{Zhang2025b,
      title={LimiX: Unleashing Structured-Data Modeling Capability for Generalist Intelligence}, 
      author={Xingxuan Zhang and Gang Ren and Han Yu and Hao Yuan and Hui Wang and Jiansheng Li and Jiayun Wu and Lang Mo and Li Mao and Mingchao Hao and Ningbo Dai and Renzhe Xu and Shuyang Li and Tianyang Zhang and Yue He and Yuanrui Wang and Yunjia Zhang and Zijing Xu and Dongzhe Li and Fang Gao and Hao Zou and Jiandong Liu and Jiashuo Liu and Jiawei Xu and Kaijie Cheng and Kehan Li and Linjun Zhou and Qing Li and Shaohua Fan and Xiaoyu Lin and Xinyan Han and Xuanyue Li and Yan Lu and Yuan Xue and Yuanyuan Jiang and Zimu Wang and Zhenlei Wang and Peng Cui},
      year={2025},
      eprint={2509.03505},
      archivePrefix={arXiv},
      primaryClass={cs.LG},
}

@inproceedings{Devlin2019,
  title={Bert: Pre-training of deep bidirectional transformers for language understanding},
  author={Devlin, Jacob and Chang, Ming-Wei and Lee, Kenton and Toutanova, Kristina},
  booktitle={Proceedings of the 2019 conference of the North American chapter of the association for computational linguistics: human language technologies, volume 1 (long and short papers)},
  pages={4171--4186},
  year={2019}
}

@inproceedings{Yin2024,
  title={Effective Causal Discovery under Identifiable Heteroscedastic Noise Model},
  author={Yin, Naiyu and Gao, Tian and Yu, Yue and Ji, Qiang},
  booktitle={Proceedings of the AAAI Conference on Artificial Intelligence},
  volume={38},
  number={15},
  pages={16486--16494},
  year={2024}
}
\bibliographystyle{conference}

\appendix
\newpage
\section{Proofs}
\label{app:a}

We provide the proofs for the theories in the main text under the identifiability framework based on Quantile Partial Effects (QPE) \citep{Chen2026}. This framework decouples the functional descriptions in \cref{assum:1} into a form more conducive to analysis from the observational distribution, thereby circumventing the need to specify the noise terms and exact causal mechanisms. Let $F(y\!\mid\!\boldsymbol{x})$ denote the conditional cumulative distribution function (CDF) corresponding to the conditional density function $p(y\!\mid\!\boldsymbol{x})$, and $Q(y\!\mid\!\boldsymbol{x})$ denote the conditional quantile function. Furthermore, for all probability density functions $p(\boldsymbol{x})$ appearing in our analysis, we assume they always exist, are strictly positive, and are at least of class $C^k$ when their $k$-th order derivatives are involved. We define the Quantile Partial Effect as follows:

\begin{definition}[Quantile Partial Effect]\label{def:qpe}
The QPE of a random variable $Y$ given $\boldsymbol{X}$ is formulated as $\boldsymbol{\psi}(y\!\mid\!\boldsymbol{x})=\nabla_{\boldsymbol{x}}Q(y\!\mid\!\boldsymbol{x})$, where the corresponding quantile is given by $\tau=Q^{-1}(y\!\mid\!\boldsymbol{x})=F(y\!\mid\!\boldsymbol{x})$.
\end{definition}

The QPE defined in \cref{def:qpe} is a quantity that depends solely on the observational distribution and is uniquely determined by it. In the subsequent analysis, the QPE $\boldsymbol{\psi}(y\!\mid\!\boldsymbol{x})$ is expressed as a finite linear span with respect to $y$:
\begin{equation}\label{eq:linear_span}
    \boldsymbol{\psi}(y\!\mid\!\boldsymbol{x})=\sum_{i=1}^{m} c_i(\boldsymbol{x})\,b_i(y),
\end{equation}
where the finite set of functions $\boldsymbol{b}=(b_1,\dots,b_m)$, referred to as \textit{basis functions}, depends only on $y$, while the \textit{coefficient functions} $c_i$ depend only on $\boldsymbol{x}$. According to the derivations in \citep{Chen2026}, the QPEs of the observational distributions corresponding to the four mechanism cases in \cref{assum:1} can all be decoupled into this form:

\begin{table}[htbp]
    \centering
    \begin{tabular}{cccc}
        \toprule
        FCM & Functional Form & QPE Form & QPE Basis Functions \\
        \midrule
        LiNGAM & $Y=\boldsymbol{w}^\intercal \boldsymbol{X}+U$ & $\boldsymbol{w}$ & $1$ \\
        ANM    & $Y=g(\boldsymbol{X})+U$ & $\nabla g$ & $1$ \\
        HNM    & $Y=g(\boldsymbol{X})+h(\boldsymbol{X})\,U$ & $\big(\nabla g-\frac{g}{h}\nabla h\big)+\frac{\nabla h}{h}\,y$ & $1,y$ \\
        PNL    & $Y=h(g(\boldsymbol{X})+U)$ & $\nabla g\,\overline{h}$ & $\overline{h}$ \\
        \bottomrule
    \end{tabular}
    \caption{The QPEs and their basis functions corresponding to the four mechanism families. Here, $h>0$ in HNM, and $\overline{h}=h'(h^{-1})$ in PNL.}
    \label{tab:qpe}
\end{table}

This implies that the observational distribution alone is fully capable of distinguishing the functional forms of the causal mechanisms:

\theoremMechanims*
\begin{proof}
First, we show that the spaces spanned by the QPEs for the four cases in \cref{def:1} are mutually non-overlapping, as summarized in \cref{tab:qpe}.
\vspace{-0.5em}
\begin{itemize}[leftmargin=*]
\itemsep 0em
\item $\mathcal{F}=\mathcal{F}_{\text{LiNGAM}}$: For LiNGAM, its QPE is always a constant vector.
\item $\mathcal{F}=\mathcal{F}_{\text{ANM}}$: For ANM with a non-linear $g$, $\nabla g$ will never degenerate into a constant vector. Therefore, the QPE space of ANM does not overlap with that of LiNGAM.
\item $\mathcal{F}=\mathcal{F}_{\text{HNM}}$: For HNM with a non-constant $h:\mathbb{R}^d\to\mathbb{R}_+$, firstly, its QPE is guaranteed to exist (since $h>0$ holds strictly, the QPE will never be ill-defined at any location due to $h=0$). Secondly, $\frac{\nabla h}{h}$ must be a non-zero vector because $h$ is non-constant. Thus, the basis function $y$ in the QPE of HNM will not degenerate into a constant $1$. This ensures that the QPE space of HNM does not overlap with those of ANM and LiNGAM.
\item $\mathcal{F}=\mathcal{F}_{\text{PNL}}$: For PNL with a monotonic non-linear function $h$, the space where its QPE resides contains only the basis function $\overline{h}=h'(h^{-1})$. Since $h$ is non-linear, $\overline{h}$ will certainly not degenerate to be independent of $y$. Hence, its QPE basis function will not degenerate into $1$, ensuring that the QPE space of PNL does not overlap with those of ANM and LiNGAM. Furthermore, because the basis function of PNL consists solely of $\overline{h}$, rather than the two basis functions present in HNM, the QPE space of PNL does not overlap with that of HNM either.
\end{itemize}
\vspace{-0.5em}
The above analysis demonstrates that, under the four cases in \cref{def:1}, there is a mapping from the QPE $\boldsymbol{\psi}(y\!\mid\!\boldsymbol{x})$ to the mechanism family $\mathcal{F}$. Meanwhile, the observational distribution $p(\boldsymbol{X},Y)$ uniquely determines a QPE $\boldsymbol{\psi}(y\!\mid\!\boldsymbol{x})$ via its density function (by the definition of QPE). Therefore, within the prior space described in \cref{assum:1}, the observational distribution uniquely identifies a mechanism family; that is, $P(\mathcal{F}\!\mid\!p(\boldsymbol{X},Y))=1$.
\end{proof}

Next, we analyze the identifiability assumptions concerning the causal direction, which necessitates the introduction of the Wronskian determinant. For a set of multivariate functions $f_1,\dots,f_k$ where each $f_i:\mathbb{R}^d\to\mathbb{R}$, the determinant
\begin{equation*}
    W_X(f_1,\dots,f_k)=\det\left[\begin{array}{cccc}
        f_1 & f_2 & \dots & f_k \\
        \partial_x f_1 & \partial_x f_2 & \dots & \partial_x f_k \\
        \vdots & \vdots & \ddots & \vdots \\
        \partial^{k-1}_x f_1 & \partial^{k-1}_x f_2 & \dots & \partial^{k-1}_x f_k
    \end{array}\right]
\end{equation*}
is called the Wronskian determinant with respect to the covariate $X$.

\begin{assumption}\label{assum:qpe_wronk}
For a given set of basis functions $\boldsymbol{b}=(b_1,\dots,b_m)$, let $s_{i,j}(\boldsymbol{x})=\partial_{x_i}\partial_{x_j}\log p(\boldsymbol{x})$ and $\boldsymbol{\eta}_i(\boldsymbol{x})=\partial_{x_i}\big(\boldsymbol{b}(x_i)\,\partial_{x_i}\log p(\boldsymbol{x})+\boldsymbol{b}'(x_i)\big)$. Assume that:
\vspace{-0.5em}
\begin{itemize}
\setlength{\itemsep}{0em}
\item[(i)] The set of basis functions $\boldsymbol{b}=(b_1,\dots,b_m)$ is specified;
\item[(ii)] For the causal mechanism $Y=f(\boldsymbol{X},U)$, the QPE $\boldsymbol{\psi}(y\!\mid\!\boldsymbol{x})$ of the effect variable $Y$ conditioned on its parent variables $\boldsymbol{X}$ lies in the linear span of the basis functions $\boldsymbol{b}$;
\item[(iii)] The QPE $\boldsymbol{\psi}(x_i\!\mid\!\boldsymbol{x}_{-i})$ of the variable $X_i$ also lies in the linear span of the basis functions $\boldsymbol{b}$;
\item[(iv)] For any $X_j\in\boldsymbol{X}_{-i},X_k\in\boldsymbol{X}_{-j}$, there exists some $\boldsymbol{x}$ such that $W_{X_j}(s_{j,k}(\boldsymbol{x}),\boldsymbol{\eta}_{j}(\boldsymbol{x}))\ne 0$.
\end{itemize}
\end{assumption}

Specifically, for any variable $X_i$, as long as its QPE $\boldsymbol{\psi}(x_i\!\mid\!\boldsymbol{x}_{-i})$ satisfies the finite linear span form (\cref{eq:linear_span}) with respect to the given basis functions $\boldsymbol{b}$, it can be shown that $W_{X_i}(s_{i,j},\boldsymbol{\eta}_{i})\equiv 0$ for any $X_j\in\boldsymbol{X}_{-i}$ and any $\boldsymbol{x}$ \citep[Theorem 3.6 and Corollary 3.8]{Chen2026}. \Cref{assum:qpe_wronk} couples this finite linear span form with the cause-and-effect relationship, and rules out coincidences where this necessary condition also identically holds for other variables $X_k\in\boldsymbol{X}_{-i}$.
\begin{lemma}[{\citealt[Corollary 3.8]{Chen2026}}]\label{lem:id_wronk}
If \cref{assum:qpe_wronk} holds, then the variable $X_i$ satisfying condition (iii) in \cref{assum:qpe_wronk} is the unique effect variable.
\end{lemma}
This lemma provides a pathway for a unified analysis of identifiability from observational distributions to causal directions and across different mechanism families. Although it can directly unify LiNGAM, ANM, and HNM, it cannot be directly applied to PNL. This is because, without any additional constraints, the basis functions in PNL remain dependent on the post-nonlinear functions, which violates the premise of specified basis functions (condition (i) in \cref{assum:qpe_wronk}). To address this, we introduce another assumption to restrict PNL:
\begin{assumption}\label{assum:finite_basis}
For the post-nonlinear function $h$ in PNL, assume that there is only a finite number of possible functions $\overline{h}=h'(h^{-1})$ within the prior space.
\end{assumption}
In practice, this assumption is highly likely to hold within the artificially constructed prior spaces for pre-training. Since the training only proceeds for a finite number of steps, only a finite number of post-nonlinear functions $h$ are sampled from the PNL mechanisms. Consequently, the total number of $\overline{h}$ involved in this prior space can be considered finite. Under \cref{assum:finite_basis}, we can incorporate PNL into the same framework as \cref{lem:id_wronk}, thereby proving the following:
\begin{corollary}\label{cor:fcm_id}
Assuming that condition (ii) in \cref{assum:1}, condition (iv) in \cref{assum:qpe_wronk}, and \cref{assum:finite_basis} hold simultaneously, the variable $X_i$ satisfying condition (iii) in \cref{assum:qpe_wronk} is the unique effect variable.
\end{corollary}

\begin{proof}
As shown in \cref{tab:qpe}, the basis functions for LiNGAM, ANM, and HNM are completely determined (the basis $\boldsymbol{b}$ are $1$, $1$, and $(1,y)$, respectively). Furthermore, \cref{assum:finite_basis} postulates that there are only finitely many possible $\overline{h}$, denoted by $(\overline{h}_1,\dots,\overline{h}_k)$. We can then construct $\boldsymbol{b}=(1,y,\overline{h}_1,\dots,\overline{h}_k)$ such that LiNGAM, ANM, HNM, and the PNL model satisfying \cref{assum:finite_basis} all reside within the same finite-dimensional space spanned by the given basis $\boldsymbol{b}$. Consequently, conditions (i), (ii), and (iii) in \cref{assum:qpe_wronk} are all satisfied. The proof is then completed by applying \cref{lem:id_wronk}.
\end{proof}

Based on the identifiability of the effect variables (i.e., the leaf nodes in a DAG), we can recursively prove the identifiability of the graph from the observational distribution:

{\renewcommand\footnote[1]{}\theoremDAG*}
\begin{proof}
Assume there are $d$ variables, with indices denoted by $i\in I=\{1,\dots, d\}$. Since $\mathcal{G}$ resides in a discrete space, $P(\mathcal{G}\!\mid\!p(\boldsymbol{X}))=1$ if and only if $P(\mathcal{G}'\!\mid\!p(\boldsymbol{X}))=0$ for all $\mathcal{G}'\ne\mathcal{G}$. Therefore, it suffices to prove that the posterior probabilities of all other graphs $\mathcal{G}'\ne\mathcal{G}$ are exactly zero. Without loss of generality, since a permutation of the indices does not alter the proof, we assume the index order $1,\dots,d$ already forms a topological sort of $\mathcal{G}$.

Next, let us consider a graph $\mathcal{G}'=(I,\mathcal{E}')$ that contains at least one pair of nodes $(i,j)$ falling into one of the following three cases: (i) Superfluous forward edge: $i<j$, with $(i,j)\in\mathcal{G}'$ but $(i,j)\notin\mathcal{G}$; (ii) Missing forward edge: $i<j$, with $(i,j)\in\mathcal{G}$ but $(i,j)\notin\mathcal{G}'$; (iii) Reversed edge: $i>j$ and $(i,j)\in\mathcal{G}'$.

Subsequently, we recursively inspect node pairs in the reverse order $d,\dots,1$. Suppose there exists a pair $(i,d)$ belonging to any of the three aforementioned conditions. We can then derive a contradiction on $p(\boldsymbol{X})$ for each respective case. Briefly, given a leaf node $i$, due to the assumptions of causal sufficiency and faithfulness, conditional independencies in $p(\boldsymbol{X})$ correspond strictly to d-separations. Therefore, the Markov blanket of $X_i$ must exactly be its parents in $\mathcal{G}$, which rules out superfluous and missing forward edges. Then, according to the FCM assumption, no reverse model can yield an observational distribution that satisfies the specific directional properties, thus precluding reversed edges.

Specifically:
\vspace{-0.5em}
\begin{itemize}[leftmargin=*]
\itemsep 0em
\item \textbf{Superfluous forward edge:} According to causal sufficiency, within the prior space $\Phi$, for any $p(\boldsymbol{X})$ induced by an SCM conforming to $\mathcal{G}'$, there must exist at least one subset $\boldsymbol{X}_{I'}$ (where $I'\subset I$) such that $X_i\perp\!\!\!\perp X_d \mid\boldsymbol{X}_{I'}$ in order for $(i,d)\notin\mathcal{E}'$ to hold. However, $X_i\perp\!\!\!\perp X_d \mid\boldsymbol{X}_{I'}$ is not satisfied in $p(\boldsymbol{X})$ due to causal faithfulness, thereby contradicting the initial assumption. 
\item \textbf{Missing forward edge:} According to causal sufficiency, within the prior space $\Phi$, for any $p(\boldsymbol{X})$ induced by an SCM conforming to $\mathcal{G}'$, there does not exist any subset $\boldsymbol{X}_{I'}$ (where $I'\subset I$) such that $X_i\perp\!\!\!\perp X_d \mid\boldsymbol{X}_{I'}$ in order for $(i,d)\notin\mathcal{E}'$ to hold. However, $X_i\perp\!\!\!\perp X_d \mid\boldsymbol{X}_{I'}$ indeed exists in $p(\boldsymbol{X})$ due to causal faithfulness, thereby contradicting the initial assumption.
\item \textbf{Reversed edge:} If an edge with $i>d$ exists, since we have already proved the absence of missing forward edges, $\mathcal{G}'$ must contain both $(i,d)$ and $(d,i)$. This implies there exists a subset $\{i,d\}\subseteq I'\subseteq I$ such that $X_i=f_i(\boldsymbol{X}_{I'\setminus \{i\}},U_i)$ and $X_d=f_d(\boldsymbol{X}_{I'\setminus \{d\}},U_d)$ (which simultaneously induce the distribution $p(\boldsymbol{X}_{I'})$). This means there are two effect variables simultaneously in $\boldsymbol{X}_{I'}$. Based on \cref{cor:fcm_id}, this cannot hold, leading to a contradiction. Thus, the initial assumption fails.
\end{itemize}
\end{proof}
\newpage
\section{Experimental Details}
\subsection{Prior Space Construction}
\label{app:b_prior}

\paragraph{Random Graph}
The experiments involve four types of random graphs:
\vspace{-0.5em}
\begin{itemize}[leftmargin=*,itemsep=0em]
    \item \textbf{Erdős–Rényi graph}: This random graph model assumes that edges between any pair of nodes are established independently with a fixed probability,.
    \item \textbf{Barabási–Albert graph}: This model generates scale-free networks with power-law degree distributions by dynamically adding nodes and edges via a preferential attachment mechanism.
    \item \textbf{Watts–Strogatz graph}: This model constructs small-world networks with high clustering coefficients and short average path lengths by randomly rewiring edges of a regular ring lattice with a given probability.
    \item \textbf{Stochastic Block Model graph}: This model partitions network nodes into predefined community blocks and generates edges according to an independent probability matrix between and within blocks, widely used for modeling networks with prominent community structures.
\end{itemize}
\vspace{-0.5em}
During training, we only use Erdős–Rényi, Barabási–Albert, and Watts–Strogatz graphs, while the Stochastic Block Model appears exclusively in the OOD evaluation phase. The number of nodes during training is uniformly sampled from $[2, 100]$, corresponding to the number of features in the tabular dataset $\mathcal{D}$ during training. The expected graph density is sampled from $[0, 1]$, where we encourage greater coverage of sparse graphs by employing a truncated exponential distribution over $[0, 1]$ with its expected value set to $0.25$.

\paragraph{Random Noise}
During training, the exogenous variables cover Normal, Laplace, Uniform, Student's T, LogNormal, Gumbel, Exponential, $\chi^2$, Beta, and Gamma distributions. During OOD evaluation, Half-Normal and Weibull distributions are used instead. The hyperparameters for these distributions are randomly selected, configured as follows:
\vspace{-0.5em}
\begin{itemize}[leftmargin=*,itemsep=0em]
    \item \textbf{Uniform}: Uniform distribution over $[0, 1]$.
    \item \textbf{Normal}, \textbf{Laplace}, \textbf{LogNormal}, \textbf{Gumbel}: The \texttt{location} parameter $\mu$ and \texttt{scale} parameter $\sigma$ are fixed to 0 and 1, respectively.
    \item \textbf{Exponential}: The \texttt{rate} parameter $\lambda$ is uniformly sampled from $[0.5, 2.0]$.
    \item \textbf{Beta}, \textbf{Gamma}: All \texttt{concentration} parameters, including $c_0$ and $c_1$, are uniformly sampled from $[0.5, 5.0]$.
    \item \textbf{Student's T}, $\boldsymbol{\chi^2}$: The \texttt{degrees of freedom} parameter $\nu$ is uniformly sampled from $[2.1, 5.0]$.
    \item \textbf{Half-Normal}, \textbf{Weibull}: The \texttt{scale} parameter $\sigma$ and \texttt{concentration} parameter $c$ are fixed to 1.
\end{itemize}
\vspace{-0.5em}
We then sample 1D random noise variables i.i.d.\ from these randomly instantiated distributions. Subsequently, these noise samples are normalized and participate in the forward inference of the random causal mechanisms with controlled signal-to-noise ratios (SNR).

\paragraph{Random Mechanisms}
The building blocks of various causal mechanisms are randomly sampled functions $\mathbb{R}^d \to \mathbb{R}$. We employ the following random functions:
\vspace{-0.5em}
\begin{itemize}[leftmargin=*,itemsep=0em]
    \item \textbf{Random Linear}: The most basic affine transformation, defined as $f(\boldsymbol{x}) = \boldsymbol{x}^\top \boldsymbol{W} + \boldsymbol{b}$. The weight matrix $\boldsymbol{W} \in \mathbb{R}^{d \times 1}$ and bias $\boldsymbol{b} \in \mathbb{R}$ are initialized using a uniform distribution and frozen afterwards. This transformation contains no additional hyperparameters to search.
    \item \textbf{MLP (Multilayer Perceptron)}: A fully connected neural network with $L$ hidden layers, expressed as $f(\boldsymbol{x}) = \boldsymbol{W}_{L+1} \sigma(\dots \sigma(\boldsymbol{x}^\top \boldsymbol{W}_1 + \boldsymbol{b}_1) \dots) + \boldsymbol{b}_{L+1}$, where all network weights are frozen after initialization. This model involves two hyperparameters: first, the hidden layer configuration controlling depth and width (e.g., $[128, 128]$ denotes two hidden layers of equal width), sampled from combinations of 1, 2, or 3 layers with layer widths from $\{32, 64, 128, 256\}$; second, the activation function $\sigma$, discretely and uniformly sampled from the candidate set $\{\text{Tanh}, \text{Sigmoid}, \text{ReLU}, \text{GELU}, \text{SiLU}, \text{Softsign}\}$.
    \item \textbf{Periodic Functions}: A single-hidden-layer sinusoidal network based on random frequency mapping, defined as $f(\boldsymbol{x}) = \frac{1}{\sqrt{m}} \sum_{i=1}^m v_i \sin(\boldsymbol{x}^\top \boldsymbol{w}_i + b_i)$, where feature weights $\boldsymbol{w}_i \sim \mathcal{N}(0, \gamma^2 \boldsymbol{I})$, biases $b_i \sim \mathcal{U}(0, 2\pi)$, and output weights $v_i \sim \mathcal{N}(0, 1)$. The output is scaled by $1/\sqrt{m}$ to maintain numerical variance stability. This function involves two hyperparameters: the number of frequencies $m$, uniformly sampled from $[1, 200]$; and the frequency scaling factor $\gamma$ controlling the width of the frequency distribution, continuously and uniformly sampled from $[0.1, 2.0]$.
    \item \textbf{GP-RFF (Random Fourier Feature Gaussian Process)}: Approximates a Gaussian Process with an RBF kernel using random Fourier features. It is defined as $f(\boldsymbol{x}) = \sqrt{\frac{2}{m}} \sum_{i=1}^m v_i \cos(\boldsymbol{x}^\top \boldsymbol{w}_i + b_i)$, where weights $\boldsymbol{w}_i \sim \mathcal{N}(0, l^{-2}\boldsymbol{I})$, biases $b_i \sim \mathcal{U}(0, 2\pi)$, and output weights $v_i \sim \mathcal{N}(0, 1)$. This model also contains two hyperparameters: the feature count $m$, uniformly sampled from $[10, 1024]$; and the lengthscale $l$ of the RBF kernel, which dictates the variance of input projection weights ($1/l^2$) and is continuously sampled from $[0.1, 3.0]$.
    \item \textbf{Random Forest}: An ensemble model composed of $T$ random decision stumps, defined as $f(\boldsymbol{x}) = \sum_{t=1}^T \left( v^{(R)}_t \mathbb{I}(\mathcal{C}_t(\boldsymbol{x})) + v^{(L)}_t (1 - \mathbb{I}(\mathcal{C}_t(\boldsymbol{x}))) \right)$, where $\mathbb{I}(\cdot)$ is the indicator function, and prediction values at left and right leaf nodes are initialized as $v^{(L)}_t, v^{(R)}_t \sim \mathcal{N}(0, \frac{1}{T})$. Its hyperparameters include the number of trees $T$ sampled from $[1, 200]$, and the stump type determining the splitting condition $\mathcal{C}_t(\boldsymbol{x})$, which is chosen discretely between two types: axis-aligned, with splitting condition $\mathcal{C}_t(\boldsymbol{x}) := x_{j} > \tau$ ($j \sim \mathcal{U}\{1, d\}$, threshold $\tau \sim \mathcal{N}(0, 1)$); and oblique, with splitting condition $\mathcal{C}_t(\boldsymbol{x}) := \boldsymbol{x}^\top \boldsymbol{w} + b > 0$ ($\boldsymbol{w} \sim \mathcal{N}(0, \boldsymbol{I})$, $b \sim \mathcal{N}(0, 1)$).
\end{itemize}
\vspace{-0.5em}
The random function type is uniformly sampled from the available candidate set associated with the specific causal mechanism. Next, the sampled random functions are used to construct the following six types of causal mechanisms:
\vspace{-0.5em}
\begin{itemize}[leftmargin=*,itemsep=0em]
    \item \textbf{LiNGAM}: Constructed via $Y = g(\boldsymbol{X}) + U$, where $g$ is a random linear function, and $U$ is guaranteed not to be sampled from Normal noise.
    \item \textbf{ANM}: Constructed via $Y = g(\boldsymbol{X}) + U$, where $g$ is any random function (except the random linear function).
    \item \textbf{HNM}: Constructed via $Y = g(\boldsymbol{X}) + h(\boldsymbol{X})\,U$, where $g$ and $h$ are random functions, and $h$ is constrained to be non-negative via an absolute value operation after sampling.
    \item \textbf{PNL}: Constructed via $Y = h(g(\boldsymbol{X}) + U)$, where $g$ is a random function, and $h$ is a 1D invertible distortion (see the ``Post-processing" paragraph for details).
    \item \textbf{General}: Constructed via $Y = g(\boldsymbol{X}, U)$, where $U$ is treated as an input dimension of $g$, and $g$ is any random function except the random linear function to prevent degeneration into linear mechanisms.
    \item \textbf{PNL-HNM}: Constructed via $Y = h(a(\boldsymbol{X}) + b(\boldsymbol{X})\,U)$, where $a$ and $b$ are random functions, and $h$ is a 1D invertible distortion.
\end{itemize}
\vspace{-0.5em}
The causal mechanism type is uniformly sampled from an available candidate set. For example, during training, the candidate set includes all mechanisms except PNL-HNM (detailed candidate sets for testing are provided in \cref{app:b_bench}). PNL-HNM is exclusively reserved for OOD evaluations.

Next, we inject noise with controlled signal-to-noise ratios (SNR), with the SNR (in dB) uniformly sampled from $[-5\text{ dB}, 10\text{ dB}]$. Specifically, for a causal mechanism $Y = f(\boldsymbol{X}, U)$, we compute the variance of $f(\boldsymbol{X}, 0)$ based on existing endogenous variables $\boldsymbol{X}$ and zero noise, and then scale the noise variance such that the ratio of noise variance to function variance matches the sampled SNR. This strategy ensures numerical stability during data sampling while encompassing both low- and high-noise regimes, diversifying the prior space to enhance model generalization. Finally, following \citep{Ormaniec2025}, we directly normalize the resulting $Y$ after forward inference $Y = f(\boldsymbol{X}, U)$.

\paragraph{Post-processing}
Inspired by the tabular data processing approach in \cite{Hollmann2025}, we randomly select a subset of features to undergo 1D warping and discretization. We consider the following 1D invertible warping functions, which are also used in PNL and PNL-HNM:
\vspace{-0.5em}
\begin{itemize}[leftmargin=*,itemsep=0em]
    \item \textbf{Sinh-Arcsinh}: A smooth nonlinear transformation constructed via hyperbolic sine and its inverse, defined as $f(x) = \sinh(\tau \operatorname{arcsinh}(x) + \epsilon)$. The shape of the warped input distribution is controlled by two hyperparameters: the skewness parameter $\epsilon$, sampled continuously from $[-1.0, 1.0]$, which controls asymmetry; and the tail-weight parameter $\tau$, sampled continuously from $[0.5, 2.0]$, which controls tail heaviness and decay rate.
    \item \textbf{Residual Flow}: A 1D invertible flow module based on a single-hidden-layer residual network, defined as $f(x) = c_{\text{safe}}x + \boldsymbol{w}^\top \tanh(x \boldsymbol{v} + \boldsymbol{b})$. To strictly guarantee monotonicity (i.e., everywhere differentiable with a strictly positive derivative to ensure invertibility), the linear term coefficient must be constrained. Since the minimum derivative contribution of the $\tanh$ layer is $\sum_{i=1}^k \min(w_i v_i, 0)$, a safety constant $c_{\text{safe}} = \alpha - \sum_{i=1}^k \min(w_i v_i, 0)$ is introduced, where $\alpha$ is the base slope. Hidden parameters $\boldsymbol{w}, \boldsymbol{v}, \boldsymbol{b} \in \mathbb{R}^k$ are randomly initialized from $\mathcal{N}(\boldsymbol{0}, \boldsymbol{I})$ and frozen. This function contains two hyperparameters: the number of hidden components $k$, discretely sampled from $\{1, 2, \dots, 10\}$, and the base slope $\alpha$, continuously sampled from $[0.1, 2.5]$.
    \item \textbf{Asymmetric Power}: An asymmetric power function constructed around a specified split point $\mu$, enabling independent controllable nonlinear warping on both sides. Formally, $f(x) = \mu + s_r \max(x - \mu, 0)^{p_r} - s_l \max(\mu - x, 0)^{p_l}$. Under our setting, the left and right scaling factors $s_l$ and $s_r$ are fixed to $1.0$, allowing the model to focus on controlling non-linearity via polynomial exponents. The remaining three hyperparameters are independently sampled: the split point $\mu \sim \mathcal{U}[-1.0, 1.0]$, and the exponents $p_l, p_r \sim \mathcal{U}[1.0, 2.5]$.
\end{itemize}
\vspace{-0.5em}
We also consider the following discretization functions:
\vspace{-0.5em}
\begin{itemize}[leftmargin=*,itemsep=0em]
    \item \textbf{Quantile}: Discretization based on the empirical cumulative distribution of the input data $\boldsymbol{x}$. By choosing equidistant quantiles in $(0, 1)$ as thresholds, the continuous values are partitioned into $K$ discrete bins, ensuring an approximately equal sample count per bin. The hyperparameter, number of bins $K$, is uniformly sampled from the integer range $[2, 50]$.
    \item \textbf{Uniform}: Equal-width binning based on the absolute numerical range of the input data $\boldsymbol{x}$. By taking the minimum and maximum of the input tensor, the value range $[\min(\boldsymbol{x}), \max(\boldsymbol{x})]$ is linearly divided into $K$ intervals of equal width. The bin count $K$ is similarly sampled from $[2, 50]$.
    \item \textbf{Ordinal}: Order-based discretization. The high-dimensional tensor $\boldsymbol{x}$ is flattened into a 1D sequence and sorted in ascending order. Then, $K-1$ absolute values at equal index intervals are extracted as cut-off thresholds. This method relies solely on relative rankings rather than continuous distributions. The number of bins $K$ is uniformly sampled from $[2, 50]$.
\end{itemize}
\vspace{-0.5em}
In training and default evaluation configurations, the proportions of features undergoing warping and discretization each vary within $[0, 0.5]$ of total features, following a truncated exponential distribution with an expected value set to $0.25$.
\subsection{Benchmarks}
\label{app:b_bench}

\paragraph{Synthetic Benchmarks}
To evaluate performance under heterogeneous or homogeneous causal mechanisms, we construct the following synthetic benchmarks:
\vspace{-0.5em}
\begin{itemize}[leftmargin=*,itemsep=0em]
    \item \textbf{LiNGAM}: Causal mechanisms are restricted to LiNGAM, without warping or discretization.
    \item \textbf{ANM}: Causal mechanisms are restricted to ANM, without warping or discretization.
    \item \textbf{HNM}: Causal mechanisms are restricted to HNM, without warping or discretization.
    \item \textbf{PNL}: Causal mechanisms are restricted to PNL, without warping or discretization.
    \item \textbf{General}: Causal mechanisms are restricted to General, with default warping and discretization proportions.
    \item \textbf{Hetero}: Causal mechanisms follow the default configuration where all aforementioned mechanisms are selected with equal probability, with default warping and discretization proportions.
\end{itemize}
\vspace{-0.5em}
Each benchmark consists of 100 randomly sampled test cases, with sample size $n = 1000$ and feature dimension $d = 20$. Except for the specified mechanisms, all other configurations are identical to the prior space construction during training.

When evaluating specific properties, we vary only the target hyperparameter during sampling (e.g., varying only sample size $n$ when evaluating sensitivity to sample size), while keeping all other settings identical to \textbf{Hetero}.

\paragraph{Real-World Benchmarks}
The real-world benchmarks used in our experiments are detailed below:
\vspace{-0.5em}
\begin{itemize}[leftmargin=*,itemsep=0em]
    \item \textbf{Sachs} \citep{Sachs2005}: We use the version of the Sachs dataset from DAG-GNN \citep{Yu2019}, which contains 11 nodes, 20 edges, and 7,466 samples.
    \item \textbf{Causal Chamber} \citep{Gamella2025}: We use data from the light tunnel chamber under standard configuration, containing 10,000 samples and 38 variables. We remove 18 constant variables and perform causal discovery on the remaining 20 non-constant variables.
\end{itemize}
\subsection{Baselines}

\cref{app:d} presents the complete experimental results involving a total of 24 baseline methods: \vspace{-0.5em}
\begin{itemize}
\setlength{\itemsep}{0em}
\item[(i)] \textbf{Traditional search-based methods}: PC \citep{Spirtes1991}, FCI \citep{Spirtes1995}, GES \citep{Chickering2002};
\item[(ii)] \textbf{LiNGAM}: ICA-LiNGAM \citep{Shimizu2006}, Direct-LiNGAM \citep{Shimizu2011};
\item[(iii)] \textbf{ANM}: CAM \citep{Buhlmann2014}, RESIT \citep{Peters2014}, RCD \citep{Maeda2020}, SCORE \citep{Rolland2022}, DAS \citep{Montagna2023b}, NoGAM \citep{Montagna2023}, CaPS \citep{Xu2024};
\item[(iv)] \textbf{HNM}: HOST \citep{Duong2023}, ICDH \citep{Yin2024}, SkewScore \citep{Lin2025};
\item[(v)] \textbf{Optimization-based algorithms}: NOTEARS \citep{Zheng2018}, GOLEM \citep{Ng2020}, GraNDAG \citep{Lachapelle2020}, DAG-GNN \citep{Yu2019}, DAGMA \citep{Bello2022};
\item[(vi)] \textbf{Amortized algorithms}: AVICI \citep{Lorch2022}, CauScale \citep{Peng2026}, TabCausal \citep{Li2026}, FoundCause \citep{Blobaum2026}.
\end{itemize}

Among these methods, PC and FCI only output equivalence classes. For unoriented edges, following \citep{Li2026}, we orient them according to the original variable ordering in the benchmark dataset (rather than the ground-truth order). For NOTEARS and DAGMA, both linear and non-linear variants are included. For amortized algorithms, we evaluate AVICI using its public \texttt{scm-v0} checkpoint, CauScale using its public \texttt{synthetic} checkpoint, and TabCausal and FoundCause using their respective sole public checkpoints.
\subsection{Training Procedure}
\label{app:b_train}

All models, including leaf prediction and parent prediction modules, are implemented using PyTorch and trained with a fixed random seed of $42$ to ensure reproducibility. We optimize the model for a total of $500,000$ iterations using a global batch size of $64$ (in practice, to prevent out-of-memory errors, gradient accumulation is applied with a mini-batch size of $1$). The optimization process employs a peak learning rate of $1 \times 10^{-4}$ along with a weight decay coefficient of $1 \times 10^{-6}$ for regularization. To enhance numerical stability during early training, a linear learning rate warmup schedule is applied over the first $2,000$ steps. Additionally, gradient clipping with a maximum norm threshold of $2.0$ is enforced to mitigate gradient explosion.
\newpage
\section{Additional Experiments: Scaling, Robustness, and Generalization}
\label{app:c}

This appendix presents additional experimental results evaluating the scaling capability, robustness under various data perturbations, and OOD generalization ability of DAG-FM in comparison with other amortized causal discovery methods. All experiments were conducted in an inference environment with 24GB VRAM; runs that encountered out-of-memory (OOM) or other engineering execution failures are omitted. Unless otherwise specified, the configurations for the synthetic benchmarks mirror the \textbf{Hetero} setting described in \cref{app:b_bench}. Furthermore, since baseline methods are incapable of outputting causal orderings, we report topological order identification performance exclusively for DAG-FM.

\subsection{Scaling Capabilities}

\paragraph{Scaling to Large Sample Sizes}
\cref{fig:2} illustrates the performance of different amortized methods as the sample size increases (with the number of dimensions fixed at $d=20$). Without specialized engineering optimizations, both AVICI and CauScale encounter OOM errors under large sample regimes. In contrast, DAG-FM achieves superior asymptotic performance and significantly outperforms TabCausal and FoundCause, demonstrating the structural advantages of the DAG-FM architecture.

\begin{figure}[htbp]
    \centering
    \includegraphics[width=0.8\linewidth]{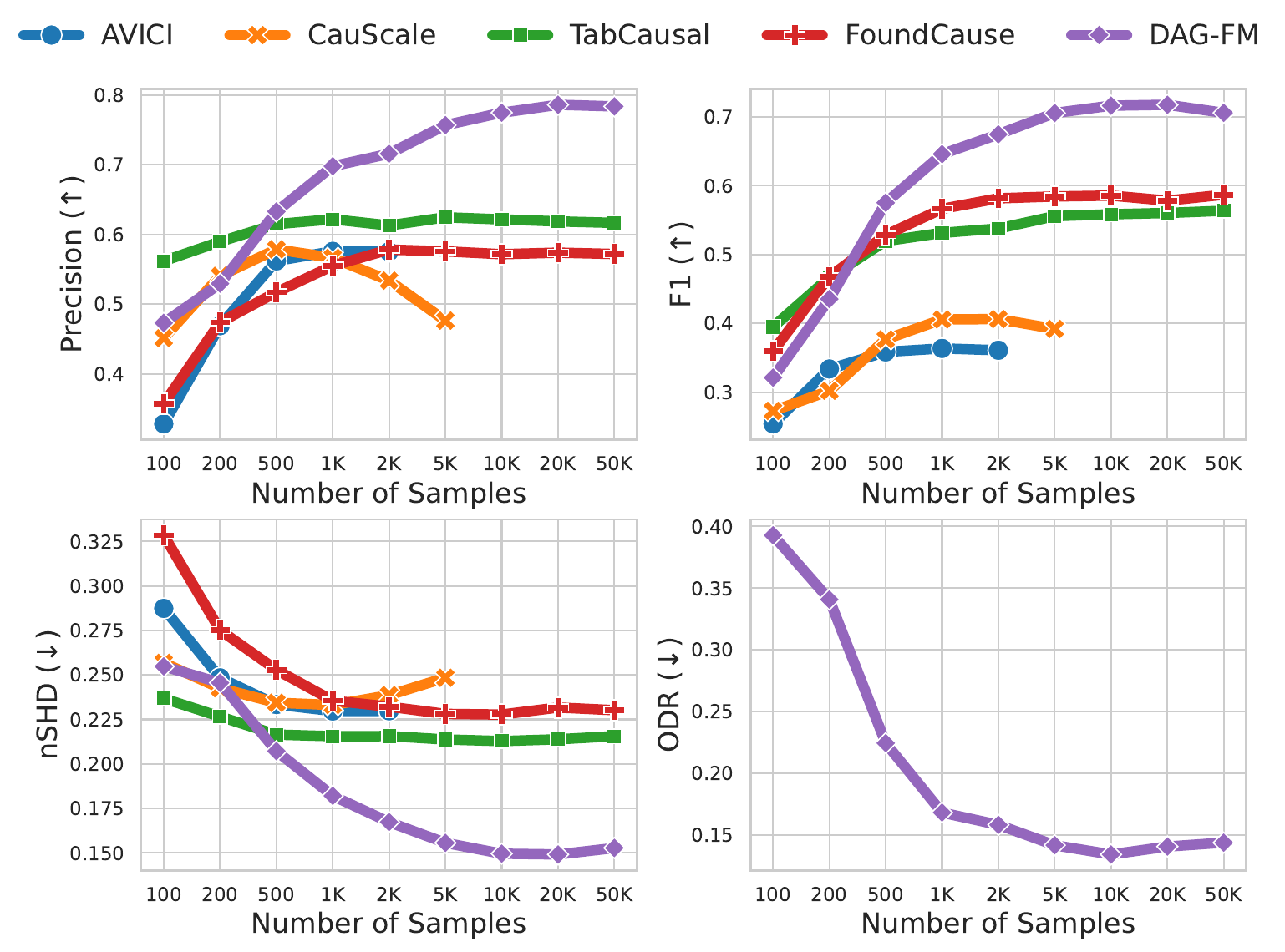}
    \caption{Performance of DAG identification (Precision, F1, nSHD) and order identification (ODR) under different numbers of samples.}
    \label{fig:2}
\end{figure}

\paragraph{Scaling to High Dimensions}
\cref{fig:3} presents the performance of amortized models as the number of features increases (with the sample size fixed at $n=1000$). AVICI and CauScale fail due to OOM errors in high-dimensional scenarios. Meanwhile, DAG-FM consistently surpasses TabCausal and FoundCause in both precision and nSHD metrics across almost all dimensional settings except $d=500$.

\begin{figure}[htbp]
    \centering
    \includegraphics[width=0.8\linewidth]{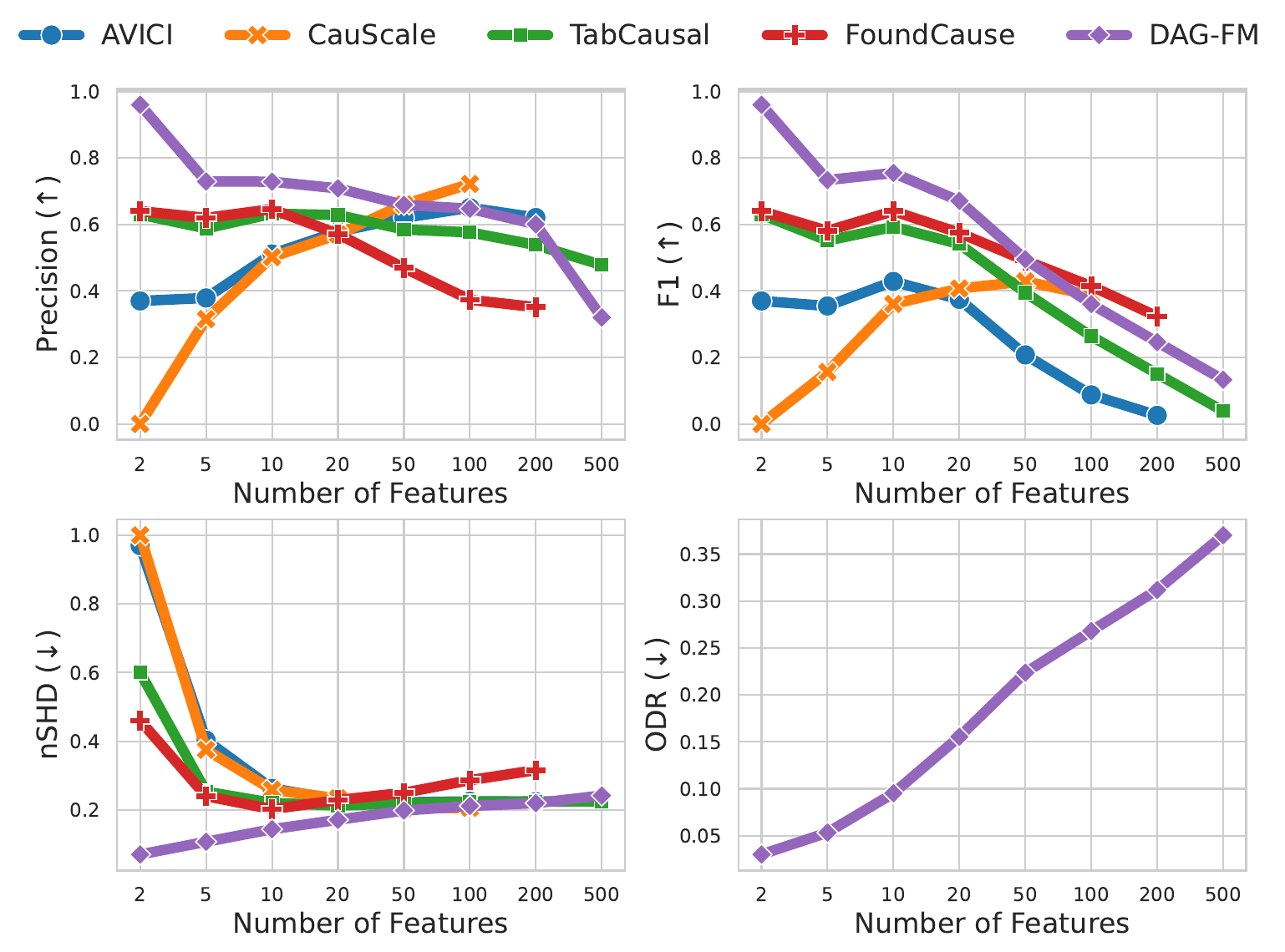}
    \caption{Performance of DAG identification (Precision, F1, nSHD) and order identification (ODR) under different numbers of features.}
    \label{fig:3}
\end{figure}
\subsection{Robustness}

\paragraph{Robustness to Graph Density}
Real-world causal graphs vary considerably in density, encompassing both sparse and highly dense structures. \cref{fig:4} displays the performance of amortized approaches as graph density increases ($\rho = 2|\mathcal{E}| / (d(d-1))$, where $|\mathcal{E}|$ denotes the number of edges in the graph). DAG-FM substantially outperforms competing methods on sparse graphs and ranks second only to FoundCause on full DAGs.

\begin{figure}[htbp]
    \centering
    \includegraphics[width=0.8\linewidth]{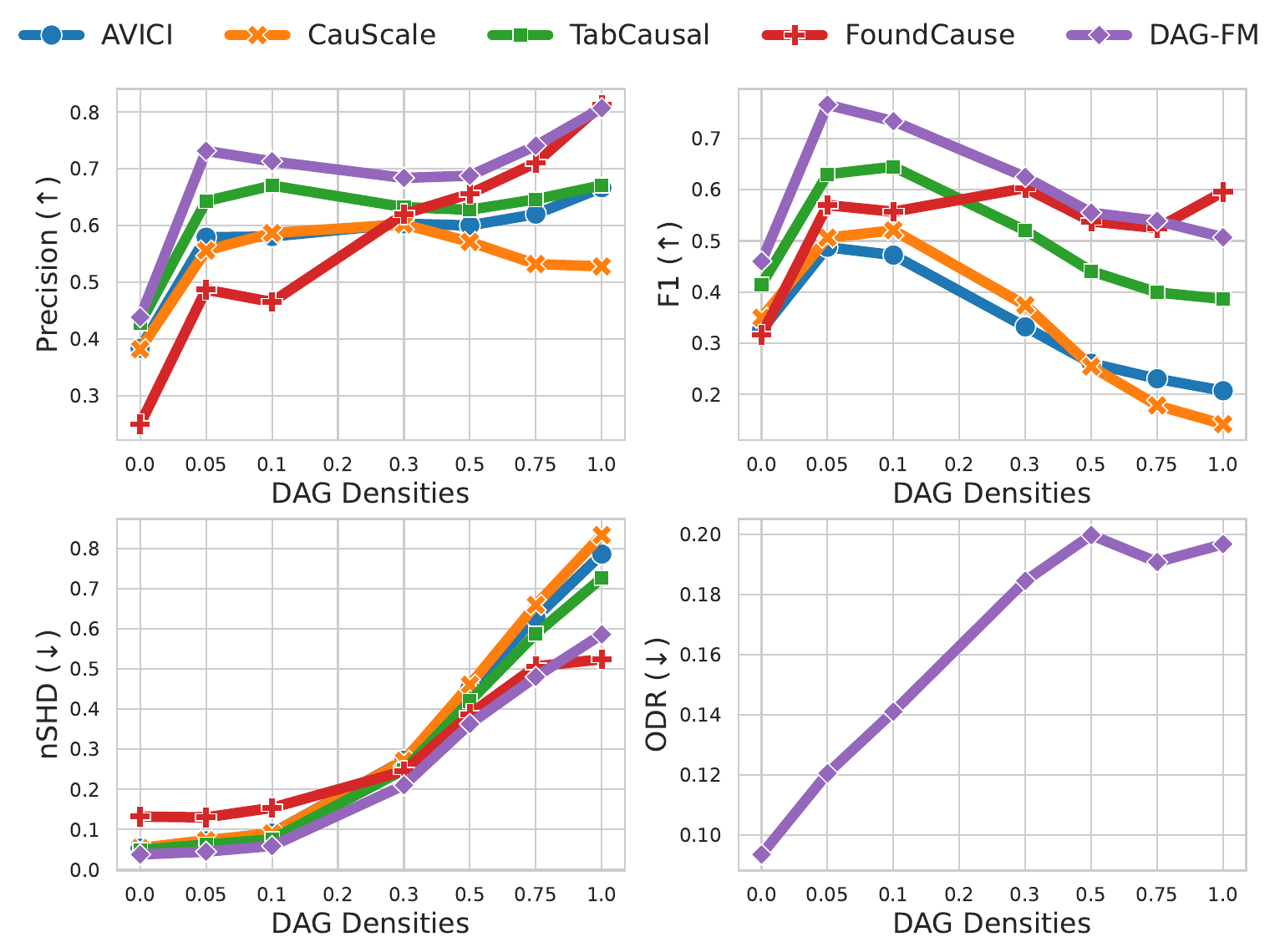}
    \caption{Performance of DAG identification (Precision, F1, nSHD) and order identification (ODR) under different densities.}
    \label{fig:4}
\end{figure}

\paragraph{Robustness to Warping Distortion}
Observational data collected in real-world environments often contains unknown distortions due to complex uncertainties, compromising compliance with underlying causal discovery assumptions. \cref{fig:5} shows model performance as the proportion of distorted variables increases, where ``proportion of warping" specifies the ratio of randomly warped variables relative to the total number of variables. DAG-FM significantly outperforms all baseline methods, demonstrating robust resilience to warping distortions.

\begin{figure}[htbp]
    \centering
    \includegraphics[width=0.8\linewidth]{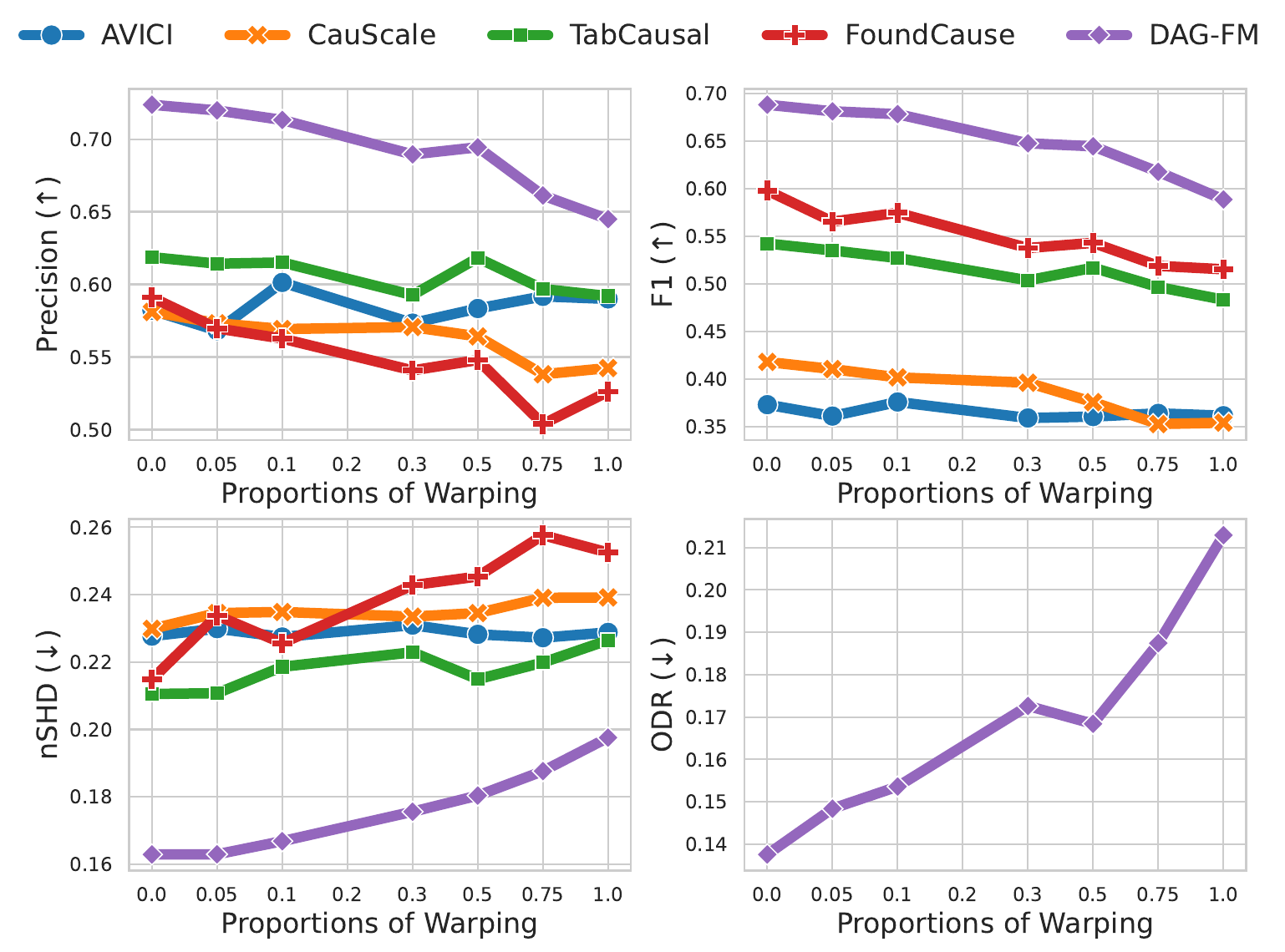}
    \caption{Performance of DAG identification (Precision, F1, nSHD) and order identification (ODR) under different proportions of warping.}
    \label{fig:5}
\end{figure}

\paragraph{Robustness to Discretization}
Another prevalent challenge in real-world datasets is mixed-type data, where features comprise both discrete and continuous variables. \cref{fig:6} compares the performance of amortized methods as the fraction of discrete variables increases, where ``proportion of discretization" indicates the ratio of discrete variables to all variables. DAG-FM outperforms competing methods across most settings, experiencing a slight performance drop only when virtually all variables are discrete, thereby confirming its robustness to discretization.

\begin{figure}[htbp]
    \centering
    \includegraphics[width=0.8\linewidth]{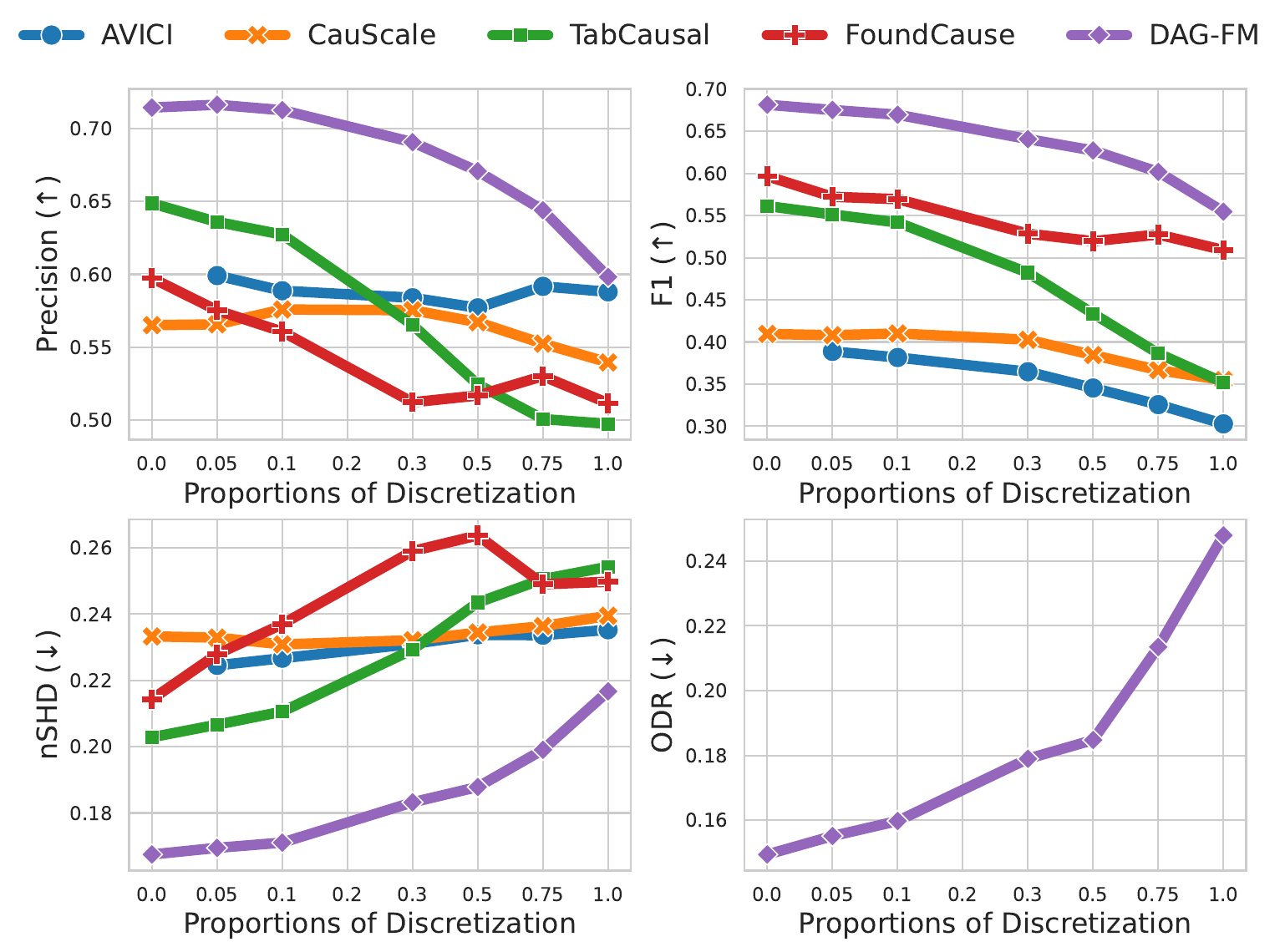}
    \caption{Performance of DAG identification (Precision, F1, nSHD) and order identification (ODR) under different proportions of discretization.}
    \label{fig:6}
\end{figure}
\subsection{Generalization to Out-of-Distribution Data}

Although amortized causal discovery algorithms enable zero-shot inference on unseen datasets, evaluating performance stability when test distributions deviate from the prior space remains crucial. \cref{fig:7} compares the performance variations of amortized methods across four distinct OOD scenarios:
\vspace{-0.5em}
\begin{itemize}[leftmargin=*]
\itemsep0em
\item \textbf{DAG}: The causal graphs of the test data are sampled from Stochastic Block Models, which are completely unseen during pre-training.
\item \textbf{Mechanism}: The underlying causal mechanisms in the test SCMs follow PNL-HNM of the form $Y = h(f(\boldsymbol{X}) + g(\boldsymbol{X})U)$, where $h$ is a non-linear invertible function. Although the theoretical identifiability of this functional form remains open, empirical evidence shows that the DAG structure can still be accurately identified.
\item \textbf{Noise}: Exogenous noise variables in the test SCMs are sampled from Half-Normal and Weibull distributions, both of which are omitted during pre-training prior space construction.
\item \textbf{Hetero}: PNL-HNM is mixed with other causal mechanisms within test datasets to systematically assess model capability under OOD heterogeneous mechanisms.
\end{itemize}
\vspace{-0.5em}

\begin{figure}[htbp]
    \centering
    \includegraphics[width=0.8\linewidth]{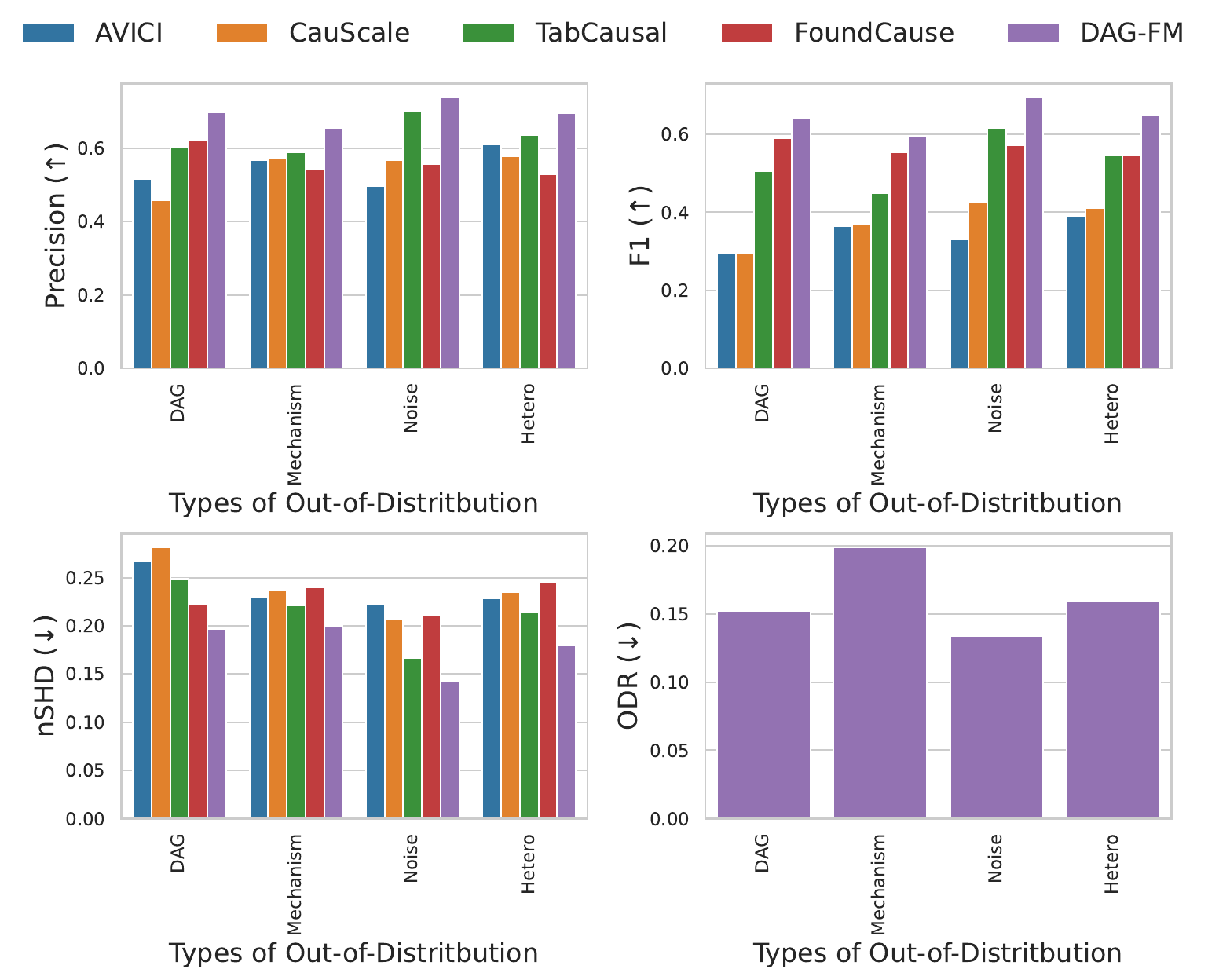}
    \caption{Performance of DAG identification (Precision, F1, nSHD) and order identification (ODR) under 4 types of OOD cases.}
    \label{fig:7}
\end{figure}

The evaluation results demonstrate that DAG-FM consistently outperforms baseline methods across all OOD scenarios, underscoring its strong generalization capabilities.
\newpage
\section{Additional Experiments: Ablation \& Full Comparison Results}
\label{app:d}

\subsection{Ablation Study on MoLE}

To demonstrate the performance gains brought by the Mixture of Causal Experts (MoLE) architecture in DAG-FM, we conduct an ablation study on the individual expert modules. Specifically, by modifying the routing logic, we isolate each individual expert by bypassing the router to exclude other experts, and compare its performance against the integrated MoLE model on both homogeneous and heterogeneous benchmarks. The results are summarized in \cref{fig:8}.

\begin{figure}[htbp]
    \centering
    \includegraphics[width=0.8\linewidth]{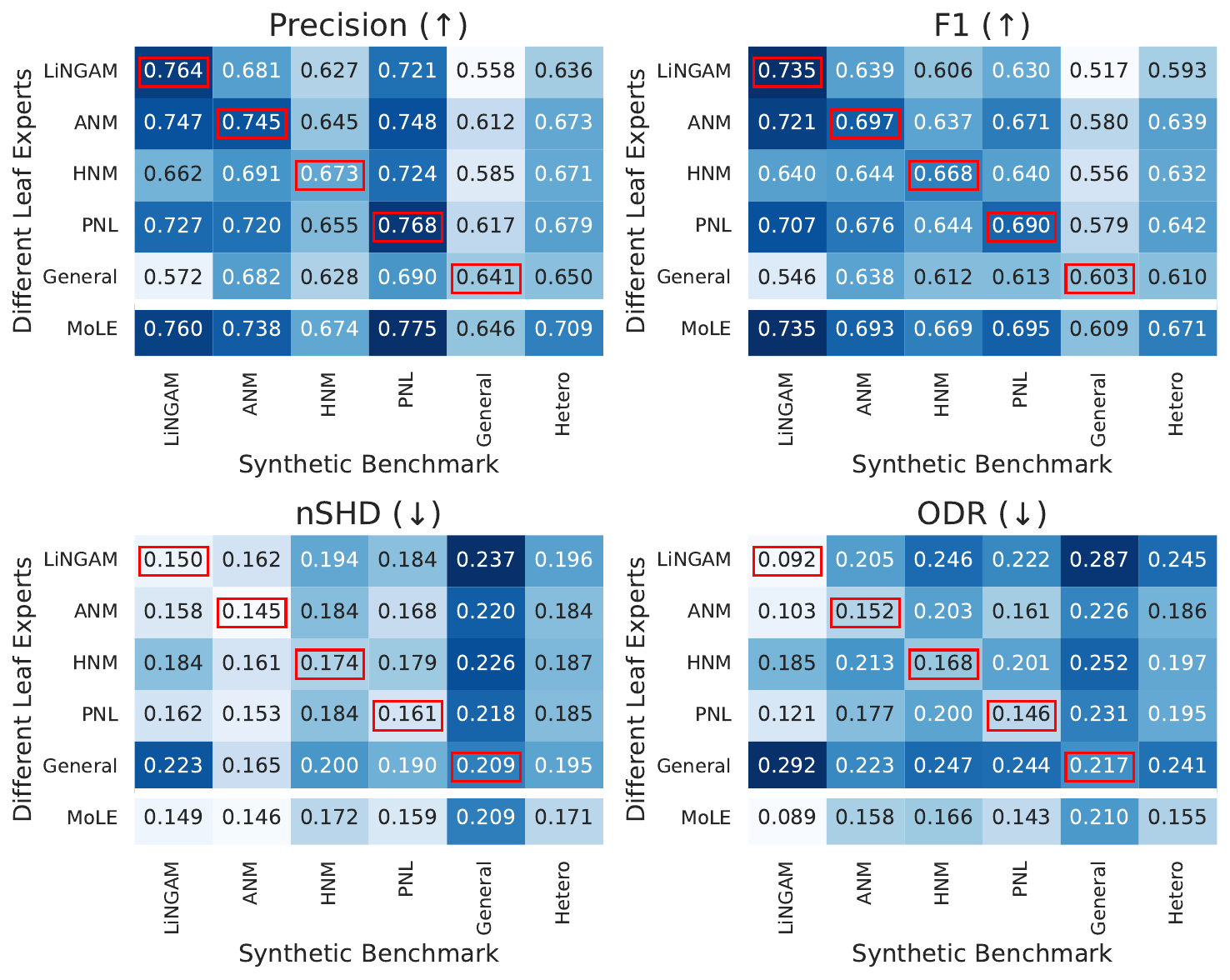}
    \caption{Performance of DAG identification (Precision, F1, nSHD) and order identification (ODR) across individual experts and the MoLE model on homogeneous and heterogeneous benchmarks.}
    \label{fig:8}
\end{figure}

The results yield three key insights: (i) In line with empirical expectations (highlighted along the diagonal red boxes), each domain-specific expert achieves the best performance among single-expert variants on its corresponding homogeneous mechanism benchmark; (ii) MoLE matches the performance of single experts under homogeneous mechanisms, confirming that the router successfully directs inputs to the appropriate experts; (iii) Under heterogeneous mechanism settings, MoLE enables DAG-FM to significantly outperform any individual expert trained on a single mechanism family.
\subsection{Ablation on Occurrence Probabilities of Heterogeneous Mechanisms}

In pre-training and the default \textbf{Hetero} synthetic benchmark, each mechanism family is sampled with equal probability. Here, we evaluate the performance of DAG-FM when mechanism families occur with non-uniform probabilities. As shown in \cref{fig:9}, from left to right, for a given target mechanism family $\mathcal{F}_i$, its occurrence probability $P(\mathcal{F}_i)$ increases from $0$ to $1$, while the remaining four mechanism families share the remaining probability equally at $(1-P(\mathcal{F}_i))/4$.

\begin{figure}[htbp]
    \centering
    \begin{subfigure}{\linewidth}
        \centering
        \includegraphics[width=0.9\linewidth]{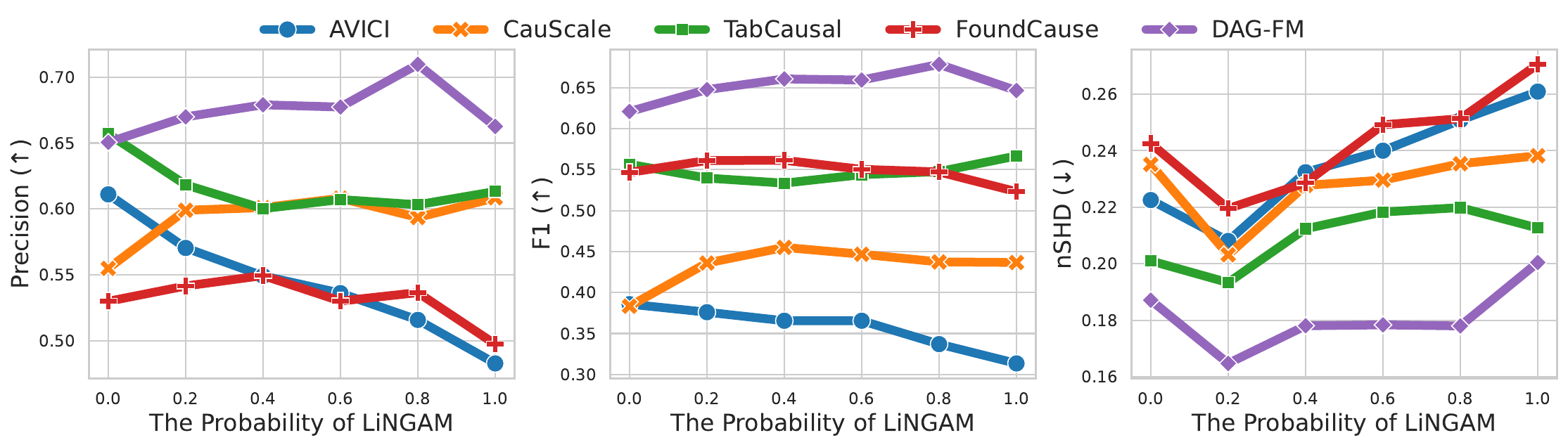}
        \caption{LiNGAM}
        \label{fig:ablation_lingam}
    \end{subfigure}
    \hfill
    \begin{subfigure}{\linewidth}
        \centering
        \includegraphics[width=0.9\linewidth]{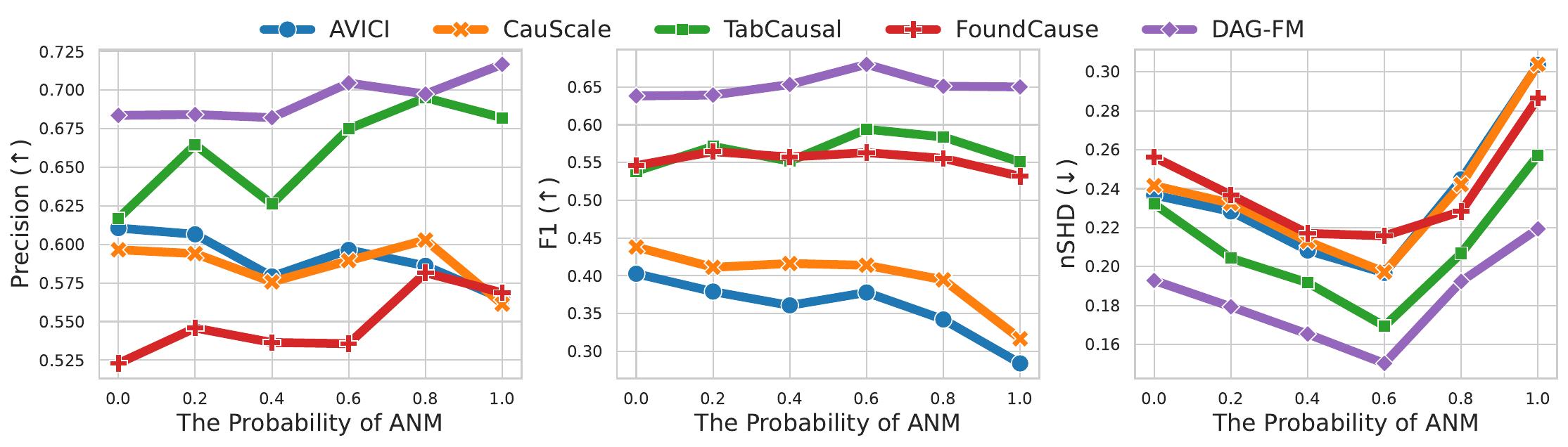}
        \caption{ANM}
        \label{fig:ablation_anm}
    \end{subfigure}
    \hfill
    \begin{subfigure}{\linewidth}
        \centering
        \includegraphics[width=0.9\linewidth]{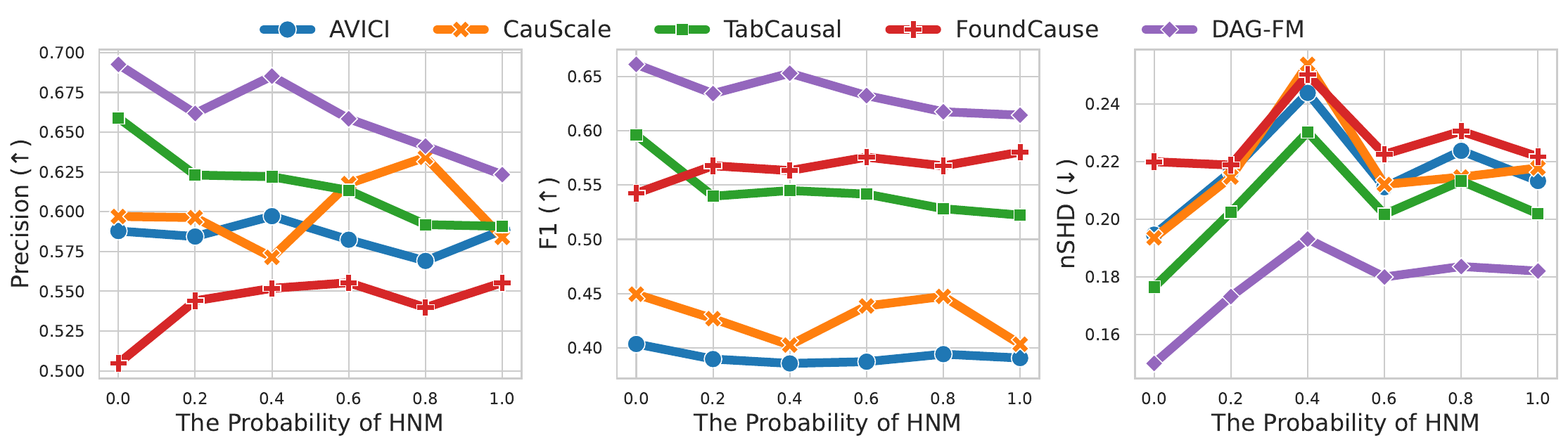}
        \caption{HNM}
        \label{fig:ablation_hnm}
    \end{subfigure}
    \hfill
    \begin{subfigure}{\linewidth}
        \centering
        \includegraphics[width=0.9\linewidth]{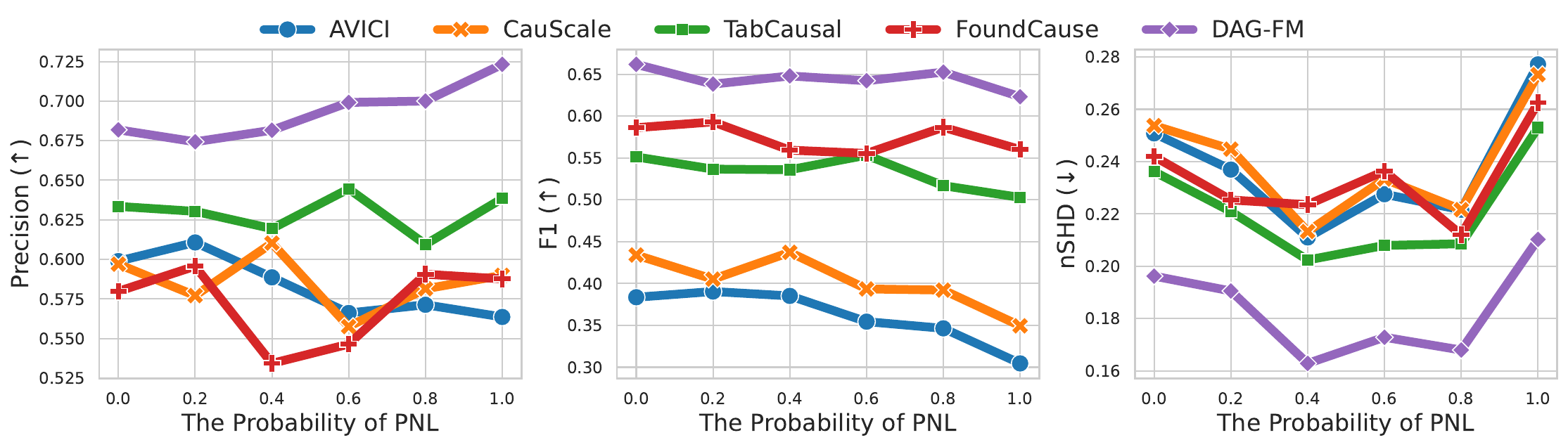}
        \caption{PNL}
        \label{fig:ablation_pnl}
    \end{subfigure}
    \hfill
    \begin{subfigure}{\linewidth}
        \centering
        \includegraphics[width=0.9\linewidth]{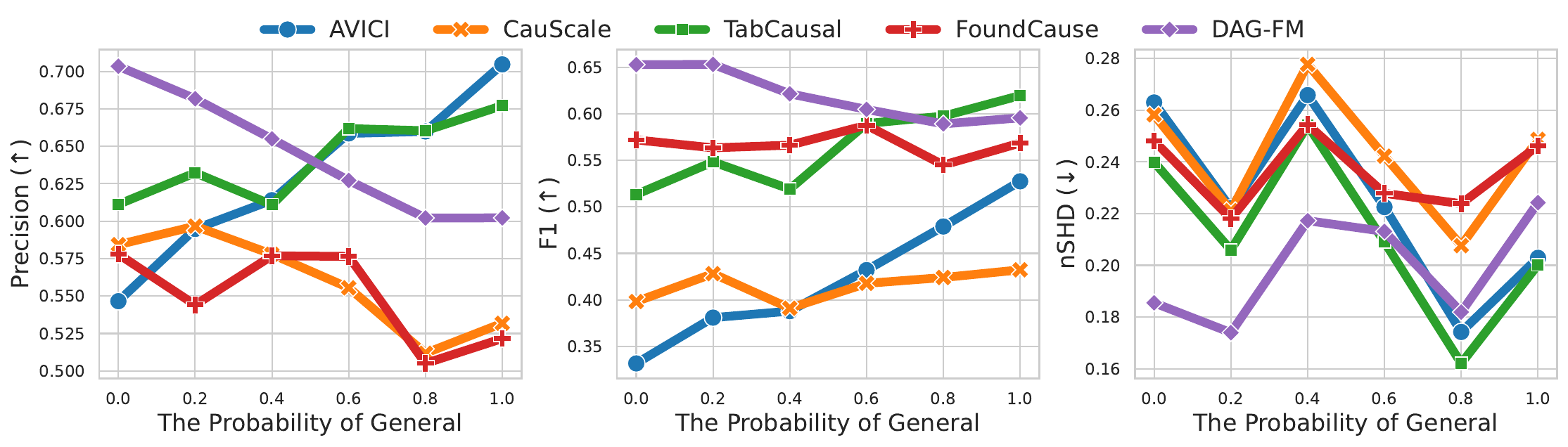}
        \caption{General}
        \label{fig:ablation_general}
    \end{subfigure}
    \caption{Performance of DAG identification (Precision, F1, nSHD) under non-uniform mechanism family occurrence probabilities: (a) LiNGAM, (b) ANM, (c) HNM, (d) PNL, and (e) General.}
    \label{fig:9}
\end{figure}

The results demonstrate that under provably identifiable settings (i.e., LiNGAM, ANM, HNM, and PNL), DAG-FM consistently outperforms baseline methods regardless of how the mechanism proportions vary, underscoring its robustness to non-uniform prior mechanism distributions. Performance degrades noticeably only under the General setting (where theoretical identifiability remains unproven), once the occurrence probability of General mechanisms exceeds $0.5$.
\subsection{Scalability Limits under Reduced Precision}

\begin{wraptable}{r}{0.5\textwidth}
    \vspace{-1.32em}
    \caption{Maximum supported sample sizes ($n$) and feature dimensions ($d$) for DAG-FM inference under a single 24GB VRAM GPU footprint across different numerical precisions.}
    \label{tab:cap}
    \vspace{-0.5em}
    \resizebox{\linewidth}{!}{\begin{tabular}{ccc}
    \toprule
    \textbf{Precision} & \textbf{Sample Size ($n$)} & \textbf{Feature Dimension ($d$)} \\
    \midrule
    Float32  & 50,000 & 20   \\
    Float32  & 20,000 & 100  \\
    Float32  & 10,000 & 160  \\
    Float32  & 1,000  & 500  \\
    \midrule
    BFloat16 & 30,000 & 100  \\
    BFloat16 & 10,000 & 500  \\
    BFloat16 & 1,000  & 1,000 \\
    \bottomrule
    \end{tabular}}
\vspace{-2em}
\end{wraptable}

\cref{tab:cap} presents the empirical computational limits of DAG-FM under a single 24GB VRAM environment during inference, including performance under BFloat16 mixed-precision.
\subsection{Full Comparison Results}

\paragraph{Synthetic Benchmarks}

\cref{tab:4,tab:5,tab:6,tab:7,tab:8,tab:9} report the performance of all 24 baseline methods across the six synthetic benchmarks (\textbf{LiNGAM}, \textbf{ANM}, \textbf{HNM}, \textbf{PNL}, \textbf{General}, and \textbf{Hetero}) described in \cref{app:b_bench}. DAG-FM consistently demonstrates a clear advantage across the vast majority of settings.

\begin{table}[htbp]
\centering
\resizebox{\linewidth}{!}{%
\begin{tabular}{l | c c c c c}
\toprule
\multicolumn{6}{c}{\textbf{DAG Performance on Synthetic LiNGAM Benchmark ($n=1000, d=20$)}} \\
\midrule
Method & Prec. $\uparrow$ & Rec. $\uparrow$ & F1 $\uparrow$ & Jac. $\uparrow$ & nSHD $\downarrow$ \\
\midrule
PC \citep{Spirtes1991} & $0.45_{0.14}$ & $0.34_{0.17}$ & $0.37_{0.15}$ & $0.23_{0.11}$ & $0.26_{0.21}$ \\
FCI \citep{Spirtes1995} & $0.45_{0.15}$ & $0.34_{0.17}$ & $0.36_{0.15}$ & $0.23_{0.11}$ & $0.26_{0.21}$ \\
GES \citep{Chickering2002} & $0.53_{0.13}$ & $\underline{0.71}_{0.21}$ & $0.60_{0.14}$ & $0.44_{0.14}$ & $0.27_{0.24}$ \\
\midrule
ICA-LiNGAM \citep{Shimizu2006} & $0.12_{0.10}$ & $0.14_{0.13}$ & $0.13_{0.11}$ & $0.07_{0.07}$ & $0.43_{0.31}$ \\
Direct-LiNGAM \citep{Shimizu2011} & $0.01_{0.02}$ & $0.01_{0.02}$ & $0.01_{0.02}$ & $0.00_{0.01}$ & $0.47_{0.38}$ \\
\midrule
CAM \citep{Buhlmann2014} & $0.20_{0.12}$ & $0.29_{0.15}$ & $0.23_{0.13}$ & $0.14_{0.08}$ & $0.42_{0.25}$ \\
RESIT \citep{Peters2014} & $0.10_{0.08}$ & $0.20_{0.15}$ & $0.13_{0.10}$ & $0.07_{0.06}$ & $0.58_{0.40}$ \\
RCD \citep{Maeda2020} & $0.03_{0.04}$ & $0.02_{0.03}$ & $0.02_{0.03}$ & $0.01_{0.02}$ & $0.40_{0.29}$ \\
SCORE \citep{Rolland2022} & $0.38_{0.15}$ & $0.51_{0.13}$ & $0.43_{0.13}$ & $0.28_{0.11}$ & $0.30_{0.19}$ \\
DAS \citep{Montagna2023b} & $0.44_{0.17}$ & $0.46_{0.14}$ & $0.44_{0.14}$ & $0.29_{0.12}$ & $0.26_{0.19}$ \\
NoGAM \citep{Montagna2023} & $0.40_{0.14}$ & $0.54_{0.13}$ & $0.45_{0.12}$ & $0.30_{0.10}$ & $0.29_{0.19}$ \\
CaPS \citep{Xu2024} & $0.39_{0.16}$ & $0.52_{0.15}$ & $0.43_{0.14}$ & $0.29_{0.12}$ & $0.30_{0.20}$ \\
\midrule
HOST \citep{Duong2023} & $0.15_{0.10}$ & $0.28_{0.12}$ & $0.18_{0.10}$ & $0.10_{0.06}$ & $0.60_{0.32}$ \\
ICDH \citep{Yin2024} & $0.43_{0.19}$ & $0.36_{0.17}$ & $0.39_{0.17}$ & $0.26_{0.14}$ & $0.26_{0.21}$ \\
SkewScore \citep{Lin2025} & $0.23_{0.10}$ & $0.40_{0.13}$ & $0.28_{0.10}$ & $0.17_{0.07}$ & $0.44_{0.27}$ \\
\midrule
NOTEARS-Linear \citep{Zheng2018} & $0.57_{0.19}$ & $0.29_{0.18}$ & $0.37_{0.19}$ & $0.24_{0.16}$ & $0.24_{0.22}$ \\
NOTEARS-Nonlinear \citep{Zheng2018} & $0.42_{0.21}$ & $0.30_{0.18}$ & $0.34_{0.19}$ & $0.22_{0.15}$ & $0.26_{0.21}$ \\
GOLEM \citep{Ng2020} & $0.41_{0.17}$ & $0.30_{0.15}$ & $0.33_{0.14}$ & $0.20_{0.11}$ & $0.27_{0.20}$ \\
GraNDAG \citep{Lachapelle2020} & $0.23_{0.30}$ & $0.03_{0.04}$ & $0.04_{0.07}$ & $0.02_{0.04}$ & $0.26_{0.21}$ \\
DAG-GNN \citep{Yu2019} & $0.44_{0.20}$ & $0.27_{0.15}$ & $0.32_{0.16}$ & $0.20_{0.13}$ & $0.26_{0.21}$ \\
DAGMA-Linear \citep{Bello2022} & $0.52_{0.22}$ & $0.33_{0.19}$ & $0.39_{0.20}$ & $0.27_{0.19}$ & $0.25_{0.21}$ \\
DAGMA-Nonlinear \citep{Bello2022} & $0.46_{0.14}$ & $0.36_{0.16}$ & $0.39_{0.14}$ & $0.25_{0.13}$ & $0.26_{0.21}$ \\
\midrule
AVICI \citep{Lorch2022} & $0.53_{0.22}$ & $0.29_{0.19}$ & $0.35_{0.19}$ & $0.23_{0.16}$ & $0.23_{0.20}$ \\
CauScale \citep{Peng2026} & $0.62_{0.18}$ & $0.40_{0.18}$ & $0.47_{0.18}$ & $0.33_{0.17}$ & $0.22_{0.21}$ \\
TabCausal \citep{Li2026} & $\textbf{0.82}_{0.18}$ & $\textbf{0.73}_{0.18}$ & $\textbf{0.77}_{0.17}$ & $\textbf{0.65}_{0.20}$ & $\textbf{0.12}_{0.14}$ \\
FoundCause \citep{Blobaum2026} & $0.60_{0.21}$ & $0.70_{0.23}$ & $0.60_{0.17}$ & $0.45_{0.19}$ & $0.22_{0.17}$ \\
\textbf{DAG-FM (ours)} & $\underline{0.76}_{0.15}$ & $\textbf{0.73}_{0.19}$ & $\underline{0.74}_{0.15}$ & $\underline{0.60}_{0.18}$ & $\underline{0.15}_{0.16}$ \\
\bottomrule
\end{tabular}%
}
\caption{DAG identification performance on the synthetic \textbf{LiNGAM} benchmark (higher is better for \texttt{Prec.}, \texttt{Rec.}, \texttt{F1}, and \texttt{Jac.}; lower is better for \texttt{nSHD}), compared with other baselines. Subscripts denote 95\% confidence intervals.}
\label{tab:4}
\end{table}

\begin{table}[htbp]
\centering
\resizebox{\linewidth}{!}{%
\begin{tabular}{l | c c c c c}
\toprule
\multicolumn{6}{c}{\textbf{DAG Performance on Synthetic ANM Benchmark ($n=1000, d=20$)}} \\
\midrule
Method & Prec. $\uparrow$ & Rec. $\uparrow$ & F1 $\uparrow$ & Jac. $\uparrow$ & nSHD $\downarrow$ \\
\midrule
PC \citep{Spirtes1991} & $0.39_{0.13}$ & $0.33_{0.14}$ & $0.34_{0.11}$ & $0.21_{0.08}$ & $0.25_{0.14}$ \\
FCI \citep{Spirtes1995} & $0.40_{0.13}$ & $0.32_{0.15}$ & $0.34_{0.11}$ & $0.21_{0.08}$ & $0.25_{0.14}$ \\
GES \citep{Chickering2002} & $0.45_{0.15}$ & $0.48_{0.17}$ & $0.44_{0.12}$ & $0.29_{0.10}$ & $0.25_{0.13}$ \\
\midrule
ICA-LiNGAM \citep{Shimizu2006} & $0.30_{0.10}$ & $0.36_{0.13}$ & $0.32_{0.09}$ & $0.19_{0.07}$ & $0.32_{0.18}$ \\
Direct-LiNGAM \citep{Shimizu2011} & $0.24_{0.10}$ & $0.27_{0.14}$ & $0.25_{0.10}$ & $0.14_{0.07}$ & $0.34_{0.20}$ \\
\midrule
CAM \citep{Buhlmann2014} & $0.46_{0.13}$ & $0.62_{0.16}$ & $0.51_{0.11}$ & $0.35_{0.09}$ & $0.23_{0.12}$ \\
RESIT \citep{Peters2014} & $0.10_{0.07}$ & $0.13_{0.10}$ & $0.11_{0.08}$ & $0.06_{0.05}$ & $0.43_{0.26}$ \\
RCD \citep{Maeda2020} & $0.05_{0.09}$ & $0.01_{0.02}$ & $0.02_{0.03}$ & $0.01_{0.02}$ & $0.27_{0.17}$ \\
SCORE \citep{Rolland2022} & $0.46_{0.13}$ & $0.59_{0.16}$ & $0.50_{0.12}$ & $0.34_{0.10}$ & $0.23_{0.13}$ \\
DAS \citep{Montagna2023b} & $0.50_{0.14}$ & $0.49_{0.18}$ & $0.48_{0.14}$ & $0.33_{0.12}$ & $0.22_{0.14}$ \\
NoGAM \citep{Montagna2023} & $0.48_{0.13}$ & $0.60_{0.16}$ & $0.52_{0.12}$ & $0.36_{0.11}$ & $0.22_{0.12}$ \\
CaPS \citep{Xu2024} & $0.45_{0.13}$ & $0.58_{0.16}$ & $0.49_{0.12}$ & $0.33_{0.10}$ & $0.24_{0.13}$ \\
\midrule
HOST \citep{Duong2023} & $0.29_{0.11}$ & $0.27_{0.13}$ & $0.27_{0.11}$ & $0.16_{0.07}$ & $0.30_{0.19}$ \\
ICDH \citep{Yin2024} & $0.51_{0.15}$ & $0.35_{0.17}$ & $0.40_{0.15}$ & $0.26_{0.14}$ & $0.22_{0.15}$ \\
SkewScore \citep{Lin2025} & $0.28_{0.13}$ & $0.38_{0.15}$ & $0.32_{0.12}$ & $0.19_{0.09}$ & $0.34_{0.17}$ \\
\midrule
NOTEARS-Linear \citep{Zheng2018} & $0.61_{0.24}$ & $0.16_{0.14}$ & $0.23_{0.16}$ & $0.14_{0.13}$ & $0.22_{0.15}$ \\
NOTEARS-Nonlinear \citep{Zheng2018} & $0.53_{0.15}$ & $0.35_{0.19}$ & $0.41_{0.17}$ & $0.27_{0.15}$ & $0.22_{0.16}$ \\
GOLEM \citep{Ng2020} & $0.47_{0.20}$ & $0.20_{0.16}$ & $0.26_{0.17}$ & $0.16_{0.14}$ & $0.23_{0.16}$ \\
GraNDAG \citep{Lachapelle2020} & $0.66_{0.34}$ & $0.12_{0.15}$ & $0.17_{0.16}$ & $0.10_{0.12}$ & $0.22_{0.16}$ \\
DAG-GNN \citep{Yu2019} & $0.44_{0.21}$ & $0.16_{0.13}$ & $0.23_{0.16}$ & $0.14_{0.11}$ & $0.23_{0.16}$ \\
DAGMA-Linear \citep{Bello2022} & $0.47_{0.23}$ & $0.18_{0.14}$ & $0.25_{0.17}$ & $0.16_{0.13}$ & $0.23_{0.16}$ \\
DAGMA-Nonlinear \citep{Bello2022} & $0.48_{0.16}$ & $0.25_{0.13}$ & $0.31_{0.13}$ & $0.19_{0.09}$ & $0.23_{0.15}$ \\
\midrule
AVICI \citep{Lorch2022} & $0.57_{0.19}$ & $0.26_{0.18}$ & $0.34_{0.19}$ & $0.22_{0.16}$ & $0.21_{0.16}$ \\
CauScale \citep{Peng2026} & $0.59_{0.21}$ & $0.31_{0.19}$ & $0.39_{0.20}$ & $0.26_{0.17}$ & $0.21_{0.16}$ \\
TabCausal \citep{Li2026} & $\underline{0.71}_{0.17}$ & $0.57_{0.19}$ & $\underline{0.62}_{0.17}$ & $\underline{0.47}_{0.18}$ & $\underline{0.17}_{0.14}$ \\
FoundCause \citep{Blobaum2026} & $0.63_{0.19}$ & $\underline{0.66}_{0.20}$ & $0.61_{0.15}$ & $0.46_{0.17}$ & $0.19_{0.13}$ \\
\textbf{DAG-FM (ours)} & $\textbf{0.74}_{0.14}$ & $\textbf{0.67}_{0.18}$ & $\textbf{0.69}_{0.14}$ & $\textbf{0.55}_{0.16}$ & $\textbf{0.15}_{0.13}$ \\
\bottomrule
\end{tabular}%
}
\caption{DAG identification performance on the synthetic \textbf{ANM} benchmark (higher is better for \texttt{Prec.}, \texttt{Rec.}, \texttt{F1}, and \texttt{Jac.}; lower is better for \texttt{nSHD}), compared with other baselines. Subscripts denote 95\% confidence intervals.}
\label{tab:5}
\end{table}

\begin{table}[htbp]
\centering
\resizebox{\linewidth}{!}{%
\begin{tabular}{l | c c c c c}
\toprule
\multicolumn{6}{c}{\textbf{DAG Performance on Synthetic HNM Benchmark ($n=1000, d=20$)}} \\
\midrule
Method & Prec. $\uparrow$ & Rec. $\uparrow$ & F1 $\uparrow$ & Jac. $\uparrow$ & nSHD $\downarrow$ \\
\midrule
PC \citep{Spirtes1991} & $0.39_{0.13}$ & $0.33_{0.17}$ & $0.32_{0.11}$ & $0.20_{0.08}$ & $0.27_{0.19}$ \\
FCI \citep{Spirtes1995} & $0.39_{0.13}$ & $0.32_{0.17}$ & $0.32_{0.11}$ & $0.20_{0.08}$ & $0.27_{0.19}$ \\
GES \citep{Chickering2002} & $0.43_{0.15}$ & $0.49_{0.19}$ & $0.43_{0.13}$ & $0.29_{0.10}$ & $0.27_{0.18}$ \\
\midrule
ICA-LiNGAM \citep{Shimizu2006} & $0.29_{0.13}$ & $0.35_{0.15}$ & $0.31_{0.13}$ & $0.19_{0.11}$ & $0.35_{0.22}$ \\
Direct-LiNGAM \citep{Shimizu2011} & $0.23_{0.10}$ & $0.26_{0.14}$ & $0.23_{0.11}$ & $0.14_{0.07}$ & $0.37_{0.27}$ \\
\midrule
CAM \citep{Buhlmann2014} & $0.29_{0.12}$ & $0.48_{0.18}$ & $0.35_{0.12}$ & $0.22_{0.08}$ & $0.37_{0.21}$ \\
RESIT \citep{Peters2014} & $0.14_{0.11}$ & $0.43_{0.16}$ & $0.19_{0.12}$ & $0.11_{0.10}$ & $0.80_{0.33}$ \\
RCD \citep{Maeda2020} & $0.05_{0.08}$ & $0.01_{0.02}$ & $0.02_{0.03}$ & $0.01_{0.01}$ & $0.29_{0.22}$ \\
SCORE \citep{Rolland2022} & $0.37_{0.14}$ & $0.56_{0.16}$ & $0.43_{0.12}$ & $0.28_{0.10}$ & $0.30_{0.16}$ \\
DAS \citep{Montagna2023b} & $0.41_{0.16}$ & $0.46_{0.18}$ & $0.42_{0.13}$ & $0.27_{0.10}$ & $0.27_{0.16}$ \\
NoGAM \citep{Montagna2023} & $0.37_{0.14}$ & $0.57_{0.15}$ & $0.43_{0.12}$ & $0.28_{0.09}$ & $0.30_{0.15}$ \\
CaPS \citep{Xu2024} & $0.38_{0.14}$ & $0.57_{0.16}$ & $0.44_{0.12}$ & $0.29_{0.10}$ & $0.30_{0.16}$ \\
\midrule
HOST \citep{Duong2023} & $0.33_{0.16}$ & $0.31_{0.15}$ & $0.31_{0.14}$ & $0.19_{0.12}$ & $0.34_{0.24}$ \\
ICDH \citep{Yin2024} & $0.47_{0.16}$ & $0.33_{0.17}$ & $0.38_{0.16}$ & $0.25_{0.14}$ & $0.25_{0.20}$ \\
SkewScore \citep{Lin2025} & $0.25_{0.12}$ & $0.40_{0.17}$ & $0.30_{0.12}$ & $0.18_{0.08}$ & $0.40_{0.23}$ \\
\midrule
NOTEARS-Linear \citep{Zheng2018} & $0.55_{0.24}$ & $0.16_{0.14}$ & $0.24_{0.17}$ & $0.15_{0.14}$ & $0.24_{0.20}$ \\
NOTEARS-Nonlinear \citep{Zheng2018} & $0.47_{0.18}$ & $0.32_{0.20}$ & $0.37_{0.19}$ & $0.24_{0.16}$ & $0.25_{0.21}$ \\
GOLEM \citep{Ng2020} & $0.44_{0.21}$ & $0.21_{0.16}$ & $0.27_{0.17}$ & $0.17_{0.12}$ & $0.26_{0.21}$ \\
GraNDAG \citep{Lachapelle2020} & $0.38_{0.35}$ & $0.06_{0.12}$ & $0.09_{0.11}$ & $0.05_{0.07}$ & $0.26_{0.20}$ \\
DAG-GNN \citep{Yu2019} & $0.45_{0.24}$ & $0.17_{0.14}$ & $0.24_{0.17}$ & $0.15_{0.13}$ & $0.26_{0.21}$ \\
DAGMA-Linear \citep{Bello2022} & $0.48_{0.23}$ & $0.20_{0.17}$ & $0.28_{0.19}$ & $0.18_{0.16}$ & $0.25_{0.21}$ \\
DAGMA-Nonlinear \citep{Bello2022} & $0.46_{0.19}$ & $0.26_{0.15}$ & $0.32_{0.16}$ & $0.20_{0.13}$ & $0.26_{0.20}$ \\
\midrule
AVICI \citep{Lorch2022} & $0.61_{0.20}$ & $0.34_{0.18}$ & $0.42_{0.19}$ & $0.28_{0.16}$ & $0.22_{0.19}$ \\
CauScale \citep{Peng2026} & $\underline{0.66}_{0.16}$ & $0.40_{0.23}$ & $0.47_{0.22}$ & $0.34_{0.21}$ & $0.22_{0.20}$ \\
TabCausal \citep{Li2026} & $0.64_{0.20}$ & $0.54_{0.19}$ & $0.57_{0.18}$ & $0.42_{0.19}$ & $\underline{0.19}_{0.16}$ \\
FoundCause \citep{Blobaum2026} & $0.62_{0.21}$ & $\textbf{0.71}_{0.21}$ & $\underline{0.62}_{0.16}$ & $\underline{0.47}_{0.18}$ & $0.21_{0.15}$ \\
\textbf{DAG-FM (ours)} & $\textbf{0.67}_{0.13}$ & $\underline{0.69}_{0.18}$ & $\textbf{0.67}_{0.13}$ & $\textbf{0.52}_{0.16}$ & $\textbf{0.17}_{0.13}$ \\
\bottomrule
\end{tabular}%
}
\caption{DAG identification performance on the synthetic \textbf{HNM} benchmark (higher is better for \texttt{Prec.}, \texttt{Rec.}, \texttt{F1}, and \texttt{Jac.}; lower is better for \texttt{nSHD}), compared with other baselines. Subscripts denote 95\% confidence intervals.}
\label{tab:6}
\end{table}

\begin{table}[htbp]
\centering
\resizebox{\linewidth}{!}{%
\begin{tabular}{l | c c c c c}
\toprule
\multicolumn{6}{c}{\textbf{DAG Performance on Synthetic PNL Benchmark ($n=1000, d=20$)}} \\
\midrule
Method & Prec. $\uparrow$ & Rec. $\uparrow$ & F1 $\uparrow$ & Jac. $\uparrow$ & nSHD $\downarrow$ \\
\midrule
PC \citep{Spirtes1991} & $0.39_{0.13}$ & $0.30_{0.13}$ & $0.32_{0.11}$ & $0.20_{0.08}$ & $0.28_{0.19}$ \\
FCI \citep{Spirtes1995} & $0.39_{0.13}$ & $0.30_{0.13}$ & $0.32_{0.11}$ & $0.20_{0.08}$ & $0.28_{0.19}$ \\
GES \citep{Chickering2002} & $0.45_{0.14}$ & $0.50_{0.19}$ & $0.45_{0.13}$ & $0.30_{0.11}$ & $0.27_{0.19}$ \\
\midrule
ICA-LiNGAM \citep{Shimizu2006} & $0.29_{0.09}$ & $0.34_{0.11}$ & $0.30_{0.08}$ & $0.18_{0.06}$ & $0.35_{0.23}$ \\
Direct-LiNGAM \citep{Shimizu2011} & $0.23_{0.09}$ & $0.25_{0.11}$ & $0.23_{0.09}$ & $0.14_{0.06}$ & $0.37_{0.26}$ \\
\midrule
CAM \citep{Buhlmann2014} & $0.35_{0.11}$ & $0.51_{0.14}$ & $0.40_{0.10}$ & $0.26_{0.08}$ & $0.33_{0.21}$ \\
RESIT \citep{Peters2014} & $0.14_{0.07}$ & $0.29_{0.11}$ & $0.18_{0.08}$ & $0.10_{0.05}$ & $0.61_{0.33}$ \\
RCD \citep{Maeda2020} & $0.09_{0.11}$ & $0.02_{0.03}$ & $0.04_{0.04}$ & $0.02_{0.02}$ & $0.30_{0.22}$ \\
SCORE \citep{Rolland2022} & $0.38_{0.12}$ & $0.52_{0.14}$ & $0.42_{0.11}$ & $0.27_{0.09}$ & $0.30_{0.18}$ \\
DAS \citep{Montagna2023b} & $0.42_{0.14}$ & $0.44_{0.15}$ & $0.41_{0.12}$ & $0.27_{0.10}$ & $0.28_{0.18}$ \\
NoGAM \citep{Montagna2023} & $0.38_{0.12}$ & $0.52_{0.15}$ & $0.43_{0.11}$ & $0.28_{0.09}$ & $0.30_{0.18}$ \\
CaPS \citep{Xu2024} & $0.38_{0.12}$ & $0.52_{0.14}$ & $0.43_{0.11}$ & $0.28_{0.09}$ & $0.30_{0.18}$ \\
\midrule
HOST \citep{Duong2023} & $0.37_{0.14}$ & $0.33_{0.13}$ & $0.33_{0.12}$ & $0.21_{0.09}$ & $0.31_{0.21}$ \\
ICDH \citep{Yin2024} & $0.46_{0.15}$ & $0.32_{0.15}$ & $0.36_{0.14}$ & $0.23_{0.11}$ & $0.26_{0.20}$ \\
SkewScore \citep{Lin2025} & $0.28_{0.12}$ & $0.39_{0.12}$ & $0.32_{0.11}$ & $0.19_{0.08}$ & $0.37_{0.20}$ \\
\midrule
NOTEARS-Linear \citep{Zheng2018} & $0.59_{0.23}$ & $0.16_{0.13}$ & $0.24_{0.16}$ & $0.15_{0.11}$ & $0.24_{0.20}$ \\
NOTEARS-Nonlinear \citep{Zheng2018} & $0.46_{0.16}$ & $0.30_{0.17}$ & $0.35_{0.16}$ & $0.23_{0.12}$ & $0.26_{0.20}$ \\
GOLEM \citep{Ng2020} & $0.47_{0.21}$ & $0.21_{0.15}$ & $0.28_{0.17}$ & $0.17_{0.13}$ & $0.26_{0.20}$ \\
GraNDAG \citep{Lachapelle2020} & $0.38_{0.35}$ & $0.06_{0.08}$ & $0.09_{0.12}$ & $0.05_{0.07}$ & $0.26_{0.21}$ \\
DAG-GNN \citep{Yu2019} & $0.47_{0.21}$ & $0.18_{0.13}$ & $0.25_{0.16}$ & $0.15_{0.11}$ & $0.26_{0.21}$ \\
DAGMA-Linear \citep{Bello2022} & $0.48_{0.21}$ & $0.19_{0.14}$ & $0.26_{0.17}$ & $0.16_{0.13}$ & $0.26_{0.21}$ \\
DAGMA-Nonlinear \citep{Bello2022} & $0.46_{0.17}$ & $0.25_{0.13}$ & $0.31_{0.14}$ & $0.19_{0.10}$ & $0.26_{0.20}$ \\
\midrule
AVICI \citep{Lorch2022} & $0.58_{0.18}$ & $0.26_{0.15}$ & $0.34_{0.16}$ & $0.22_{0.13}$ & $0.24_{0.19}$ \\
CauScale \citep{Peng2026} & $0.62_{0.16}$ & $0.33_{0.18}$ & $0.41_{0.19}$ & $0.27_{0.15}$ & $0.23_{0.20}$ \\
TabCausal \citep{Li2026} & $\underline{0.66}_{0.15}$ & $0.50_{0.20}$ & $0.55_{0.17}$ & $0.40_{0.17}$ & $\underline{0.21}_{0.19}$ \\
FoundCause \citep{Blobaum2026} & $0.64_{0.21}$ & $\textbf{0.68}_{0.19}$ & $\underline{0.62}_{0.15}$ & $\underline{0.46}_{0.17}$ & $\underline{0.21}_{0.16}$ \\
\textbf{DAG-FM (ours)} & $\textbf{0.78}_{0.11}$ & $\underline{0.66}_{0.19}$ & $\textbf{0.69}_{0.14}$ & $\textbf{0.55}_{0.17}$ & $\textbf{0.16}_{0.16}$ \\
\bottomrule
\end{tabular}%
}
\caption{DAG identification performance on the synthetic \textbf{PNL} benchmark (higher is better for \texttt{Prec.}, \texttt{Rec.}, \texttt{F1}, and \texttt{Jac.}; lower is better for \texttt{nSHD}), compared with other baselines. Subscripts denote 95\% confidence intervals.}
\label{tab:7}
\end{table}

\begin{table}[htbp]
\centering
\resizebox{\linewidth}{!}{%
\begin{tabular}{l | c c c c c}
\toprule
\multicolumn{6}{c}{\textbf{DAG Performance on Synthetic General Mechanism Benchmark ($n=1000, d=20$)}} \\
\midrule
Method & Prec. $\uparrow$ & Rec. $\uparrow$ & F1 $\uparrow$ & Jac. $\uparrow$ & nSHD $\downarrow$ \\
\midrule
PC \citep{Spirtes1991} & $0.35_{0.12}$ & $0.28_{0.14}$ & $0.29_{0.11}$ & $0.18_{0.08}$ & $0.30_{0.19}$ \\
FCI \citep{Spirtes1995} & $0.36_{0.12}$ & $0.28_{0.14}$ & $0.29_{0.11}$ & $0.18_{0.08}$ & $0.30_{0.19}$ \\
GES \citep{Chickering2002} & $0.37_{0.12}$ & $0.51_{0.16}$ & $0.41_{0.12}$ & $0.27_{0.10}$ & $0.34_{0.21}$ \\
\midrule
ICA-LiNGAM \citep{Shimizu2006} & $0.25_{0.08}$ & $0.39_{0.11}$ & $0.30_{0.08}$ & $0.18_{0.06}$ & $0.46_{0.27}$ \\
Direct-LiNGAM \citep{Shimizu2011} & $0.26_{0.10}$ & $0.39_{0.13}$ & $0.30_{0.10}$ & $0.18_{0.08}$ & $0.43_{0.25}$ \\
\midrule
CAM \citep{Buhlmann2014} & $0.42_{0.14}$ & $\underline{0.68}_{0.12}$ & $0.51_{0.13}$ & $0.35_{0.11}$ & $0.29_{0.16}$ \\
RESIT \citep{Peters2014} & $0.12_{0.07}$ & $0.39_{0.13}$ & $0.18_{0.08}$ & $0.10_{0.05}$ & $0.81_{0.27}$ \\
RCD \citep{Maeda2020} & $0.07_{0.12}$ & $0.01_{0.01}$ & $0.01_{0.02}$ & $0.01_{0.01}$ & $0.28_{0.20}$ \\
SCORE \citep{Rolland2022} & $0.35_{0.12}$ & $0.59_{0.16}$ & $0.44_{0.12}$ & $0.29_{0.10}$ & $0.36_{0.21}$ \\
DAS \citep{Montagna2023b} & $0.37_{0.12}$ & $0.48_{0.17}$ & $0.41_{0.11}$ & $0.26_{0.09}$ & $0.32_{0.20}$ \\
NoGAM \citep{Montagna2023} & $0.37_{0.11}$ & $0.62_{0.15}$ & $0.45_{0.11}$ & $0.30_{0.09}$ & $0.34_{0.20}$ \\
CaPS \citep{Xu2024} & $0.35_{0.11}$ & $0.60_{0.16}$ & $0.44_{0.12}$ & $0.29_{0.09}$ & $0.35_{0.21}$ \\
\midrule
HOST \citep{Duong2023} & $0.23_{0.11}$ & $0.31_{0.13}$ & $0.25_{0.10}$ & $0.15_{0.07}$ & $0.45_{0.27}$ \\
ICDH \citep{Yin2024} & $0.41_{0.13}$ & $0.38_{0.14}$ & $0.39_{0.12}$ & $0.25_{0.11}$ & $0.30_{0.21}$ \\
SkewScore \citep{Lin2025} & $0.26_{0.11}$ & $0.45_{0.14}$ & $0.32_{0.11}$ & $0.20_{0.08}$ & $0.44_{0.24}$ \\
\midrule
NOTEARS-Linear \citep{Zheng2018} & $0.46_{0.17}$ & $0.18_{0.13}$ & $0.24_{0.14}$ & $0.15_{0.12}$ & $0.26_{0.20}$ \\
NOTEARS-Nonlinear \citep{Zheng2018} & $0.44_{0.14}$ & $0.37_{0.17}$ & $0.39_{0.15}$ & $0.25_{0.13}$ & $0.28_{0.20}$ \\
GOLEM \citep{Ng2020} & $0.32_{0.14}$ & $0.23_{0.12}$ & $0.26_{0.11}$ & $0.15_{0.08}$ & $0.32_{0.21}$ \\
GraNDAG \citep{Lachapelle2020} & $0.44_{0.41}$ & $0.05_{0.08}$ & $0.08_{0.12}$ & $0.05_{0.08}$ & $0.26_{0.20}$ \\
DAG-GNN \citep{Yu2019} & $0.35_{0.18}$ & $0.21_{0.13}$ & $0.25_{0.14}$ & $0.15_{0.12}$ & $0.30_{0.21}$ \\
DAGMA-Linear \citep{Bello2022} & $0.42_{0.18}$ & $0.23_{0.14}$ & $0.29_{0.15}$ & $0.18_{0.13}$ & $0.28_{0.21}$ \\
DAGMA-Nonlinear \citep{Bello2022} & $0.42_{0.12}$ & $0.32_{0.12}$ & $0.36_{0.10}$ & $0.22_{0.08}$ & $0.28_{0.20}$ \\
\midrule
AVICI \citep{Lorch2022} & $\textbf{0.69}_{0.14}$ & $0.42_{0.18}$ & $0.50_{0.16}$ & $0.35_{0.17}$ & $\textbf{0.21}_{0.17}$ \\
CauScale \citep{Peng2026} & $0.51_{0.13}$ & $0.38_{0.17}$ & $0.42_{0.16}$ & $0.28_{0.14}$ & $0.26_{0.20}$ \\
TabCausal \citep{Li2026} & $\underline{0.65}_{0.13}$ & $0.57_{0.17}$ & $\underline{0.60}_{0.14}$ & $\underline{0.44}_{0.15}$ & $\textbf{0.21}_{0.18}$ \\
FoundCause \citep{Blobaum2026} & $0.51_{0.19}$ & $\textbf{0.69}_{0.17}$ & $0.55_{0.13}$ & $0.39_{0.13}$ & $\underline{0.25}_{0.14}$ \\
\textbf{DAG-FM (ours)} & $\underline{0.65}_{0.12}$ & $0.59_{0.18}$ & $\textbf{0.61}_{0.14}$ & $\textbf{0.45}_{0.16}$ & $\textbf{0.21}_{0.18}$ \\
\bottomrule
\end{tabular}%
}
\caption{DAG identification performance on the synthetic \textbf{General} benchmark (higher is better for \texttt{Prec.}, \texttt{Rec.}, \texttt{F1}, and \texttt{Jac.}; lower is better for \texttt{nSHD}), compared with other baselines. Subscripts denote 95\% confidence intervals.}
\label{tab:8}
\end{table}

\begin{table}[htbp]
\centering
\resizebox{\linewidth}{!}{%
\begin{tabular}{l | c c c c c}
\toprule
\multicolumn{6}{c}{\textbf{DAG Performance on Synthetic Heterogeneous Mechanism Benchmark ($n=1000, d=20$)}} \\
\midrule
Method & Prec. $\uparrow$ & Rec. $\uparrow$ & F1 $\uparrow$ & Jac. $\uparrow$ & nSHD $\downarrow$ \\
\midrule
PC \citep{Spirtes1991} & $0.40_{0.13}$ & $0.32_{0.14}$ & $0.34_{0.11}$ & $0.21_{0.08}$ & $0.27_{0.21}$ \\
FCI \citep{Spirtes1995} & $0.40_{0.13}$ & $0.32_{0.14}$ & $0.33_{0.12}$ & $0.21_{0.08}$ & $0.27_{0.21}$ \\
GES \citep{Chickering2002} & $0.45_{0.15}$ & $0.53_{0.19}$ & $0.48_{0.15}$ & $0.32_{0.13}$ & $0.27_{0.21}$ \\
\midrule
ICA-LiNGAM \citep{Shimizu2006} & $0.28_{0.10}$ & $0.36_{0.12}$ & $0.31_{0.09}$ & $0.19_{0.06}$ & $0.36_{0.24}$ \\
Direct-LiNGAM \citep{Shimizu2011} & $0.22_{0.10}$ & $0.25_{0.13}$ & $0.22_{0.10}$ & $0.13_{0.06}$ & $0.39_{0.29}$ \\
\midrule
CAM \citep{Buhlmann2014} & $0.34_{0.13}$ & $0.52_{0.15}$ & $0.40_{0.12}$ & $0.26_{0.09}$ & $0.32_{0.21}$ \\
RESIT \citep{Peters2014} & $0.14_{0.08}$ & $0.32_{0.12}$ & $0.19_{0.09}$ & $0.10_{0.05}$ & $0.62_{0.34}$ \\
RCD \citep{Maeda2020} & $0.08_{0.12}$ & $0.02_{0.03}$ & $0.03_{0.05}$ & $0.02_{0.02}$ & $0.30_{0.23}$ \\
SCORE \citep{Rolland2022} & $0.38_{0.14}$ & $0.53_{0.15}$ & $0.43_{0.12}$ & $0.28_{0.10}$ & $0.29_{0.17}$ \\
DAS \citep{Montagna2023b} & $0.43_{0.16}$ & $0.45_{0.15}$ & $0.42_{0.13}$ & $0.28_{0.10}$ & $0.26_{0.18}$ \\
NoGAM \citep{Montagna2023} & $0.39_{0.14}$ & $\underline{0.55}_{0.14}$ & $0.45_{0.12}$ & $0.30_{0.10}$ & $0.28_{0.17}$ \\
CaPS \citep{Xu2024} & $0.38_{0.13}$ & $0.54_{0.14}$ & $0.43_{0.12}$ & $0.28_{0.09}$ & $0.29_{0.17}$ \\
\midrule
HOST \citep{Duong2023} & $0.28_{0.13}$ & $0.29_{0.12}$ & $0.28_{0.11}$ & $0.16_{0.08}$ & $0.35_{0.26}$ \\
ICDH \citep{Yin2024} & $0.44_{0.16}$ & $0.33_{0.16}$ & $0.37_{0.14}$ & $0.23_{0.11}$ & $0.26_{0.21}$ \\
SkewScore \citep{Lin2025} & $0.28_{0.11}$ & $0.42_{0.14}$ & $0.33_{0.11}$ & $0.20_{0.08}$ & $0.39_{0.25}$ \\
\midrule
NOTEARS-Linear \citep{Zheng2018} & $0.54_{0.20}$ & $0.17_{0.12}$ & $0.24_{0.14}$ & $0.15_{0.10}$ & $0.24_{0.21}$ \\
NOTEARS-Nonlinear \citep{Zheng2018} & $0.44_{0.16}$ & $0.30_{0.16}$ & $0.35_{0.15}$ & $0.22_{0.11}$ & $0.26_{0.21}$ \\
GOLEM \citep{Ng2020} & $0.41_{0.18}$ & $0.22_{0.13}$ & $0.27_{0.14}$ & $0.16_{0.10}$ & $0.27_{0.22}$ \\
GraNDAG \citep{Lachapelle2020} & $0.39_{0.40}$ & $0.05_{0.08}$ & $0.08_{0.10}$ & $0.04_{0.06}$ & $0.26_{0.21}$ \\
DAG-GNN \citep{Yu2019} & $0.40_{0.20}$ & $0.18_{0.14}$ & $0.24_{0.16}$ & $0.15_{0.12}$ & $0.27_{0.22}$ \\
DAGMA-Linear \citep{Bello2022} & $0.43_{0.18}$ & $0.19_{0.12}$ & $0.25_{0.14}$ & $0.15_{0.10}$ & $0.26_{0.22}$ \\
DAGMA-Nonlinear \citep{Bello2022} & $0.45_{0.15}$ & $0.28_{0.13}$ & $0.33_{0.13}$ & $0.21_{0.10}$ & $0.26_{0.20}$ \\
\midrule
AVICI \citep{Lorch2022} & $0.58_{0.19}$ & $0.30_{0.19}$ & $0.37_{0.20}$ & $0.25_{0.16}$ & $0.23_{0.21}$ \\
CauScale \citep{Peng2026} & $0.57_{0.18}$ & $0.34_{0.19}$ & $0.41_{0.19}$ & $0.27_{0.16}$ & $0.23_{0.21}$ \\
TabCausal \citep{Li2026} & $\underline{0.63}_{0.18}$ & $0.50_{0.19}$ & $0.54_{0.17}$ & $0.39_{0.17}$ & $\underline{0.21}_{0.21}$ \\
FoundCause \citep{Blobaum2026} & $0.57_{0.22}$ & $\textbf{0.66}_{0.18}$ & $\underline{0.57}_{0.16}$ & $\underline{0.42}_{0.16}$ & $0.23_{0.17}$ \\
\textbf{DAG-FM (ours)} & $\textbf{0.71}_{0.14}$ & $\textbf{0.66}_{0.19}$ & $\textbf{0.67}_{0.15}$ & $\textbf{0.52}_{0.16}$ & $\textbf{0.17}_{0.17}$ \\
\bottomrule
\end{tabular}%
}
\caption{DAG identification performance on the synthetic \textbf{Hetero} benchmark (higher is better for \texttt{Prec.}, \texttt{Rec.}, \texttt{F1}, and \texttt{Jac.}; lower is better for \texttt{nSHD}), compared with other baselines. Subscripts denote 95\% confidence intervals.}
\label{tab:9}
\end{table}

\paragraph{Real-world Benchmarks}

\cref{tab:10} summarizes the DAG identification performance of all 24 baseline methods on two widely used real-world benchmarks. DAG-FM achieves state-of-the-art performance across these datasets.

\begin{table*}[htbp]
\centering
\resizebox{\linewidth}{!}{%
\begin{tabular}{l | c c c c c | c c c c c}
\toprule
\multicolumn{11}{c}{\textbf{{DAG Performance on Real-world Benchmarks}}} \\
\midrule
\multirow{2}{*}{Method} & \multicolumn{5}{c|}{\textbf{{Sachs}} ($n=7466, d=11$)} & \multicolumn{5}{c}{\textbf{{Causal Chamber}} ($n=10\text{{k}}, d=20$)} \\
 & Prec. $\uparrow$ & Rec. $\uparrow$ & F1 $\uparrow$ & Jac. $\uparrow$ & nSHD $\downarrow$ & Prec. $\uparrow$ & Rec. $\uparrow$ & F1 $\uparrow$ & Jac. $\uparrow$ & nSHD $\downarrow$ \\
\midrule
PC \citep{Spirtes1991} & 0.20 & 0.25 & 0.22 & 0.12 & 0.64 & 0.00 & 0.00 & 0.00 & 0.00 & 0.34 \\
FCI \citep{Spirtes1995} & 0.21 & 0.25 & 0.23 & 0.13 & 0.62 & 0.00 & 0.00 & 0.00 & 0.00 & 0.32 \\
GES \citep{Chickering2002} & 0.23 & 0.45 & 0.31 & 0.18 & 0.76 & 0.50 & \textbf{0.62} & 0.55 & 0.38 & 0.19 \\
\midrule
ICA-LiNGAM \citep{Shimizu2006} & 0.20 & 0.45 & 0.28 & 0.16 & 0.85 & 0.05 & 0.05 & 0.05 & 0.03 & 0.41 \\
Direct-LiNGAM \citep{Shimizu2011} & 0.19 & 0.35 & 0.25 & 0.14 & 0.76 & 0.05 & 0.05 & 0.05 & 0.03 & 0.40 \\
\midrule
CAM \citep{Buhlmann2014} & 0.29 & \textbf{0.75} & \textbf{0.42} & \textbf{0.27} & 0.75 & \multicolumn{5}{c}{Out of Memory} \\
RESIT \citep{Peters2014} & 0.22 & \underline{0.60} & 0.32 & 0.19 & 0.93 & \multicolumn{5}{c}{Out of Memory} \\
RCD \citep{Maeda2020} & 0.00 & 0.00 & 0.00 & 0.00 & \underline{0.36} & \multicolumn{5}{c}{Out of Memory} \\
SCORE \citep{Rolland2022} & 0.20 & 0.45 & 0.28 & 0.16 & 0.85 & \multicolumn{5}{c}{Out of Memory} \\
DAS \citep{Montagna2023b} & 0.22 & 0.35 & 0.27 & 0.16 & 0.69 & \multicolumn{5}{c}{Out of Memory} \\
NoGAM \citep{Montagna2023} & \multicolumn{5}{c|}{Out of Memory} & \multicolumn{5}{c}{Out of Memory} \\
CaPS \citep{Xu2024} & 0.07 & 0.15 & 0.09 & 0.05 & 1.09 & 0.11 & 0.15 & 0.13 & 0.07 & 0.43 \\
\midrule
HOST \citep{Duong2023} & \multicolumn{5}{c|}{Out of Memory} & \multicolumn{5}{c}{Out of Memory} \\
ICDH \citep{Yin2024} & 0.20 & 0.20 & 0.20 & 0.11 & 0.58 & 0.06 & 0.05 & 0.06 & 0.03 & 0.35 \\
SkewScore \citep{Lin2025} & 0.23 & 0.50 & 0.31 & 0.19 & 0.80 & 0.07 & 0.13 & 0.09 & 0.05 & 0.53 \\
\midrule
NOTEARS-Linear \citep{Zheng2018} & 0.33 & 0.10 & 0.15 & 0.08 & 0.40 & 0.15 & 0.05 & 0.08 & 0.04 & 0.25 \\
NOTEARS-Nonlinear \citep{Zheng2018} & 0.07 & 0.05 & 0.06 & 0.03 & 0.58 & 0.18 & 0.10 & 0.13 & 0.07 & 0.28 \\
GOLEM \citep{Ng2020} & 0.10 & 0.10 & 0.10 & 0.05 & 0.65 & 0.07 & 0.05 & 0.06 & 0.03 & 0.34 \\
GraNDAG \citep{Lachapelle2020} & 0.00 & 0.00 & 0.00 & 0.00 & 0.40 & \textbf{1.00} & 0.03 & 0.05 & 0.03 & 0.20 \\
DAG-GNN \citep{Yu2019} & \multicolumn{5}{c|}{Out of Memory} & \multicolumn{5}{c}{Out of Memory} \\
DAGMA-Linear \citep{Bello2022} & 0.20 & 0.10 & 0.13 & 0.07 & 0.47 & 0.04 & 0.03 & 0.03 & 0.02 & 0.32 \\
DAGMA-Nonlinear \citep{Bello2022} & 0.25 & 0.20 & 0.22 & 0.12 & 0.51 & \multicolumn{5}{c}{Out of Memory} \\
\midrule
AVICI \citep{Lorch2022} & \multicolumn{5}{c|}{Out of Memory} & \multicolumn{5}{c}{Out of Memory} \\
CauScale \citep{Peng2026} & \multicolumn{5}{c|}{Out of Memory} & \multicolumn{5}{c}{Out of Memory} \\
TabCausal \citep{Li2026} & 0.38 & 0.30 & \underline{0.33} & \underline{0.20} & 0.44 & 0.71 & 0.51 & \underline{0.60} & \underline{0.43} & \underline{0.14} \\
FoundCause \citep{Blobaum2026} & \underline{0.50} & 0.25 & \underline{0.33} & \underline{0.20} & \underline{0.36} & 0.53 & \underline{0.59} & 0.56 & 0.39 & 0.19 \\
\textbf{DAG-FM (ours)} & \textbf{0.54} & 0.35 & \textbf{0.42} & \textbf{0.27} & \textbf{0.35} & \underline{0.77} & 0.51 & \textbf{0.62} & \textbf{0.44} & \textbf{0.13} \\
\bottomrule
\end{tabular}%
}
\caption{DAG identification performance on two real-world benchmarks (higher is better for \texttt{Prec.}, \texttt{Rec.}, \texttt{F1}, and \texttt{Jac.}; lower is better for \texttt{nSHD}), compared with other baselines.}
\label{tab:10}
\end{table*}


\end{document}